\documentclass{article} 
\usepackage{iclr2020_conference,times}

\usepackage{amsmath,amsfonts,bm}

\def\eqref#1{equation~\ref{#1}}

\def\1{\bm{1}}

\DeclareMathAlphabet{\mathsfit}{\encodingdefault}{\sfdefault}{m}{sl}
\SetMathAlphabet{\mathsfit}{bold}{\encodingdefault}{\sfdefault}{bx}{n}

\usepackage{amsmath}

\usepackage[utf8]{inputenc} 
\usepackage[T1]{fontenc}    
\usepackage{hyperref}       
\usepackage{url}            
\usepackage{booktabs}       
\usepackage{amsfonts}       
\usepackage{nicefrac}       
\usepackage{microtype}      
\usepackage{graphicx}
\usepackage{mathabx}
\usepackage{longtable}
\usepackage{amsthm}   
\usepackage{stackrel}

\usepackage{bm} 
\usepackage{subfig}
\usepackage{float}  
\usepackage{multirow}
\usepackage{xcolor,colortbl} 
\usepackage{color}  
\usepackage[linesnumbered,ruled]{algorithm2e}

\providecommand{\realnum}{\mathbb{R}}
\providecommand{\naturalnum}{\mathbb{N}}
\providecommand{\bmcal}[1]{\bm{\mathcal{#1}}}
\providecommand{\algcol}[1]{\textcolor{blue}{#1}}
 \providecommand{\matnot}[1]{_{[{#1}]}}  
\providecommand{\citep}{\cite} 
\providecommand{\citet}{\cite}
\newtheorem{lemma}{Lemma}
\newtheorem{claim}{Claim}
\newcommand{\supplementary}{appendix}

\iclrfinalcopy

\title{PolyGAN: High-Order Polynomial Generators}

\author{
Grigorios Chrysos${}^1$,
Stylianos Moschoglou${}^1$,
Yannis Panagakis${}^{1,2}$,
Stefanos Zafeiriou${}^1$\\
Imperial College London${}^1$\\
Middlesex University${}^2$\\
\texttt{\{g.chrysos, s.moschoglou, i.panagakis, s.zafeiriou\}@imperial.ac.uk}
}

\newcommand{\modelname}{PolyGAN}
\newcommand{\modelone}{Coupled CP decomposition}
\newcommand{\modeltwo}{Coupled nested CP decomposition}

\begin{document}

\maketitle

\begin{abstract}

Generative Adversarial Networks (GANs) have become the gold standard when it comes to learning generative models for high-dimensional distributions. Since their advent, numerous variations of GANs have been introduced in the literature, primarily focusing on utilization of novel loss functions, optimization/regularization strategies and network architectures. In this paper, we turn our attention to the generator and investigate the use of high-order polynomials as an alternative class of universal function approximators. Concretely, we propose PolyGAN, where we model the data generator by means of a high-order polynomial whose unknown parameters are naturally represented by high-order tensors. We introduce two tensor decompositions that significantly reduce the number of parameters and show how they can be efficiently implemented by hierarchical neural networks that only employ linear/convolutional blocks. We exhibit for the first time that by using our approach a GAN generator can approximate the data distribution without using \textsl{any} activation functions. Thorough experimental evaluation on both synthetic and real data (images and 3D point clouds) demonstrates the merits of PolyGAN against the state of the art.

\end{abstract}

\section{Introduction}

Generative Adversarial Networks (GANs) are currently one of the most popular lines of research in machine learning. Research on GANs mainly revolves around: (a) how to achieve faster and/or more accurate convergence (e.g., by studying different loss functions~\citep{nowozin2016f, arjovsky17principled, mao2017least} or regularization schemes~\citep{odena2018generator, miyato2018spectral, gulrajani2017improved}), and (b) how to design different hierarchical neural networks architectures composed of linear and non-linear operators that can effectively model high-dimensional distributions (e.g., by progressively training large networks \citep{karras2017progressive} or by utilizing deep ResNet type of networks as generators~\citep{brock2019large}). 

Even though hierarchical deep networks are efficient universal approximators for the class of continuous compositional functions \citep{mhaskar2016learning}, the non-linear \textbf{activation functions} pose difficulties in their theoretical analysis,  understanding, and interpretation. For instance, as illustrated in  \cite{arora2018convergence}, element-wise non-linearities pose a challenge on proving convergence, especially in an adversarial learning setting \citep{ji2018minimax}. Consequently, several methods, e.g., \cite{saxe2013exact, hardt2016identity, laurent2018deep, lampinen2018analytic}, focus only on linear models (with respect to the weights) in order to be able to rigorously  analyze the neural network dynamics, the residual design principle, local extrema and generalization error, respectively. Moreover, as stated in the recent in-depth comparison of many different GAN training schemes \citep{lucic2018gans}, the improvements may mainly arise from a higher computational budget and tuning and not from fundamental architectural choices.

In this paper, we depart from the choice of hierarchical neural networks that involve activation functions and investigate for the first time in the literature of GANs the use of high-order polynomials as an alternative class of universal function approximators for data generator functions. This choice is motivated by the strong evidence provided by the \textit{Stone–Weierstrass theorem}~\citep{stone1948generalized}, which states that every continuous function defined on a closed interval can be uniformly approximated as closely as desired by a polynomial function. Hence, we propose to model the vector-valued generator function $\bm{G}(\bm{z}): \realnum^{d} \to \realnum^{o}$ by a high-order multivariate polynomial of the latent vector $\bm{z}$, whose unknown parameters are naturally represented by high-order tensors. 

However, the number of parameters required to accommodate all higher-order correlations of the latent vector explodes with  the desired order of the polynomial and the
dimension of the latent vector. To alleviate this issue and at the same time capture interactions of parameters across different orders of approximation in a hierarchical manner, we cast polynomial parameters estimation as a coupled tensor factorization~\citep{papalexakis2016turbo, sidiropoulos2017tensor} that jointly factorizes all the polynomial parameters tensors. To this end, we introduce two specifically tailored coupled canonical polyadic (CP)-type of decompositions with shared factors. The proposed coupled decompositions of the parameters tensors result into two different hierarchical structures (i.e., architectures of neural network decoders) that do not involve \textsl{any} activation function, providing an intuitive way of generating samples with an increasing level of detail. This is pictorially shown in Figure~\ref{fig:powergan_skip_injections}. The result of the proposed \modelname{} using a fourth-order polynomial approximator is shown in  Figure~\ref{fig:powergan_skip_injections} (a), while Figure~\ref{fig:powergan_skip_injections} (b) shows the corresponding generation when removing the fourth-order power from the generator.

Our contributions are summarized as follows:
\begin{itemize}
    \item We model the data generator with a high-order polynomial. Core to our approach is to cast polynomial parameters estimation as a coupled tensor factorization with shared factors. To this end, we develop two coupled tensor decompositions and demonstrate how those two derivations result in different neural network architectures involving only linear (e.g., convolution) units. This approach reveals links between high-order polynomials, coupled tensor decompositions and network architectures.
    \item We experimentally verify that the resulting networks can learn to approximate functions with analytic expressions.
    \item We show how the proposed networks can be used with linear blocks, i.e., without utilizing activation functions, to synthesize high-order intricate signals, such as images.
    \item We demonstrate that by incorporating activation functions to the derived polynomial-based architectures,
    \modelname{} improves upon  three different GAN architectures, namely DCGAN~\citep{radford2015unsupervised}, SNGAN~\citep{miyato2018spectral} and SAGAN~\citep{zhang2018self}.
\end{itemize}

\begin{figure*}[!h]
    \centering
    \subfloat[]{\includegraphics[width=0.45\linewidth]{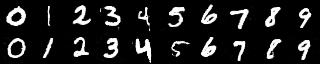}}
  \quad
    \subfloat[]{\includegraphics[width=0.45\linewidth]{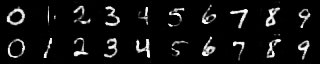}\hspace{2.9mm}}
\caption{Generated samples by an instance of the proposed \modelname. (a) Generated samples using a fourth-order polynomial and (b) the corresponding generated samples when removing the terms that correspond to the fourth-order. As evidenced, by extending the polynomial terms, \modelname{} generates samples with an increasing level of detail.}
\label{fig:powergan_skip_injections}
\end{figure*}

  \section{Method}
\label{sec:polygan_method}
In this Section, we investigate the use of a polynomial expansion as a function approximator for the data generator in the context of GANs. To begin with, we introduce the notation in Section~\ref{sec:polygan_notation}.
In Section~\ref{sec:polygan_decomposition}, we introduce two different polynomials models along with specifically tailored coupled tensor factorizations for the efficient estimation of their parameters.

\subsection{Preliminaries and notation}
\label{sec:polygan_notation}
Matrices (vectors) are denoted by uppercase (lowercase) boldface letters e.g., $\bm{X}$, ($\bm{x}$).  Tensors are denoted by calligraphic letters, e.g.,  $\bmcal{X}$. The
\textit{order} of a tensor is the number of indices needed to
address its elements. Consequently, each element of an $M$th-order tensor $\bmcal{X}$ is addressed by $M$ indices, i.e., $(\bmcal{X})_{i_{1}, i_{2}, \ldots, i_{M}} \doteq x_{i_{1}, i_{2}, \ldots, i_{M}}$.

 The \textit{mode-$m$ unfolding} of a tensor $\bmcal{X} \in
 \realnum^{I_1 \times I_2 \times \cdots \times I_M}$ maps
 $\bmcal{X}$ to a matrix $\bm{X}_{(m)} \in \realnum^{I_{m}
 \times \bar{I}_{m}}$ with $\bar{I}_{m}= \prod_{k=1 \atop k  \neq m}^M I_k $ such
 that the tensor element $x_{i_1, i_2, \ldots, i_M}$ is
 mapped to the matrix element $x_{i_{m}, j}$ where
 $j=1 + \sum_{k=1 \atop k \neq m}^M (i_k - 1) J_k$ with $J_k =
\prod_{n =1 \atop n \neq m}^{k-1} I_n $. 
The \textit{mode-$m$ vector product} of $\bmcal{X}$ with a
vector $\bm{u} \in \realnum^{I_m}$, denoted by
$\bmcal{X} \times_{n} \bm{u} \in \realnum^{I_{1}\times
I_{2}\times\cdots\times I_{n-1}  \times I_{n+1} \times
\cdots \times I_{N}} $, results in a tensor of order $M-1$:
\begin{equation}\label{E:Tensor_Mode_n}
(\bmcal{X} \times_{m} \bm{u})_{i_1, \ldots, i_{m-1}, i_{m+1},
\ldots, i_{M}} = \sum_{i_m=1}^{I_m} x_{i_1, i_2, \ldots, i_{M}} u_{i_m}.
\end{equation}
Furthermore,  we denote
$\bmcal{X} \times_{1} \bm{u}^{(1)} \times_{2} \bm{u}^{(2)} \times_{3}  \cdots \times_{M} \bm{u}^{(M)}  \doteq 
\bmcal{X} \prod_{m=1}^m \times_{m} \bm{u}^{(m)}$.

The \textit{Khatri-Rao} product (i.e., column-wise Kronecker product) of matrices $\bm{A} \in \realnum^{I \times N}$
and $\bm{B} \in \realnum^{J \times N}$ is
denoted by $\bm{A} \odot \bm{B}$ and yields a matrix of
dimensions $(IJ)\times N$.  The Hadamard product of $\bm{A} \in \realnum^{I \times N}$
and $\bm{B} \in \realnum^{I \times N}$ is defined as $\bm{A} * \bm{B}$ and is equal to ${A}_{(i, j)} {B}_{(i, j)}$ for the $(i, j)$ element.

The \textit{CP decomposition}~\citep{kolda2009tensor,Sidiropoulos:16}  factorizes a tensor into a sum of component rank-one tensors. An $M$th-order tensor $\bmcal{X} \in \realnum^{I_{1} \times
 I_{2} \times \cdots \times I_{M}}$ has \textit{rank}-$1$, when it is decomposed as the outer product of $M$ vectors $\{  \bm{u}^{ (m)}  \in \realnum^{I_{m}} \}_{m=1}^M$. That is,  $\bmcal{X} =  
 \bm{u}^{(1)}  \circ \bm{u}^{(2)}  \circ \cdots \circ \bm{u}^{(M)} \doteq  \bigcirc_{m =1}^M \bm{x}^{(m)}$, where $\circ$ denotes for the vector outer product. Consequently, the rank-$R$ CP decomposition of an $M$th-order tensor $\bmcal{X}$ is written as:
 \begin{equation}\label{E:CP}
\bmcal{X}  \doteq [\![ \bm{U}\matnot{1}, \bm{U}\matnot{2}, \ldots, \bm{U}\matnot{M}  ]\!] =  \sum_{r=1}^R \bm{u}_r^{(1)}  \circ \bm{u}_r^{(2)}  \circ \cdots \circ \bm{u}_r^{(M)},
\end{equation}
where the factor matrices $\big\{ \bm{U}\matnot{m} = [\bm{u}_1^{(m)},\bm{u}_2^{(m)}, \cdots, \bm{u}_R^{(m)} ] \in \mathbb{R}^{I_m \times R} \big\}_{m=1}^{M}$ collect 
the vectors from the rank-one components. By considering the mode-$1$ unfolding of $\bmcal{X}$, the CP decomposition can be written in matrix form as \citep{kolda2009tensor}:
 \begin{equation}
 \label{eq:polygan_cp_unfolding}
\bmcal{X}_{(1)}  =  \bm{U}\matnot{1} \bigg( \bm{U}\matnot{M} \odot \bm{U}\matnot{M-1} \odot \cdots \odot \bm{U}\matnot{2} \bigg)^T \doteq \bm{U}\matnot{1} \bigg( \bigodot_{m = M}^{2} \bm{U}\matnot{m}\bigg)^T
\end{equation}
 
More details on tensors and multilinear operators can be found in \citet{kolda2009tensor,Sidiropoulos:16}.

\subsection{High-order polynomial generators}
\label{sec:polygan_decomposition}

GANs typically consist of two deep networks, namely a generator $G$ and a discriminator $D$. $G$ is a decoder (i.e., a function approximator of the sampler of the target distribution) which receives as input a random noise vector $\bm{z} \in \realnum^{d}$ and outputs a sample $\bm{x} = G(\bm{z}) \in \realnum^{o}$. $D$ receives as input both $G(\bm{z})$ and real samples and tries to differentiate the fake and the real samples. During training, both $G$ and $D$ compete against each other till they reach an ``equilibrium'' \citep{goodfellow2014generative}. In practice, both the generator and the discriminator are modeled as deep neural networks, involving composition of linear and non-linear operators~\citep{radford2015unsupervised}.

\begin{table}[h]
\caption{Nomenclature}
\label{tbl:polygan_primary_symbols}
\begin{tabular}{|c | c | c|}
\toprule
Symbol 	& Dimension(s) 		&	Definition \\
\midrule
$n, N$ 		            & $\naturalnum$		            &	Polynomial term order, total approximation order. \\
$k$ 		            & $\naturalnum$		            & Rank of the decompositions. \\
$\bm{z}$            & $\realnum^d$                      & Input to the polynomial approximator, i.e., generator. \\
$\bm{C}, \bm{\beta}$ 		    & $\realnum^{o\times k}, \realnum^{o}$		        &	Parameters in both decompositions. \\
$\bm{A}\matnot{n}, \bm{S}\matnot{n}, \bm{B}\matnot{n}$          &       $\realnum^{d\times k}, \realnum^{k\times k}, \realnum^{\omega\times k}$     & Matrix parameters in the hierarchical decomposition.\\
$\odot, *$          &   -       & Khatri-Rao product, Hadamard product. \\
 \hline
\end{tabular}
\end{table}

In this paper, we focus on the generator. Instead of modeling the generator as a composition of linear and non-linear functions, we assume that each generated pixel $x_i = (G(\bm{z}))_i$ may be expanded as a $N^{th}$ order polynomial\footnote{With an $N^{th}$ order polynomial we can approximate any smooth function \citep{stone1948generalized}.} in $\bm{z}$. That is,

\begin{equation}
  x_i = (G(\bm{z}))_i = \beta_i + {\bm{w}_i^{[1]}}^T\bm{z} + \bm{z}^T \bm{W}_i^{[2]}\bm{z} + \bmcal{W}_i^{[3]}\times_1\bm{z}\times_2\bm{z}\times_3\bm{z} + \cdots + \bmcal{W}_i^{[N]}\prod_{n=1}^N \times_{n} \bm{z},
\label{eq:elemet_wise_exp}
\end{equation}
where the scalar $\beta_i$, and the set of tensors $\big\{\bmcal{W}_i^{[n]} \in \realnum^{\prod_{m=1}^n\times_m d}\big\}_{n=1}^N$ are the parameters of the polynomial expansion associated to each output of the generator, e.g., pixel. Clearly, when $n=1$, the weights are $d$-dimensional vectors; when $n=2$, the weights, i.e., $\bm{W}_i^{[2]}$, form a $d\times d$ matrix. For higher orders of approximation, i.e., when $n\geq3$, the weights are $n^{th}$ order tensors.

By stacking the parameters for all pixels, we define the parameters $\bm{\beta} \doteq [\beta_1, \beta_2, \ldots, \beta_o ]^T \in \realnum^o$ and $\big\{\bmcal{W}^{[n]} \in  \realnum^{o\times \prod_{m=1}^{n}\times_m d}\big\}_{n=1}^N$. Consequently, the vector-valued generator function is expressed as:
\begin{equation}
    G(\bm{z}) = \sum_{n=1}^N \bigg(\bmcal{W}^{[n]} \prod_{j=2}^{n+1} \times_{j} \bm{z}\bigg) + \bm{\beta}
    \label{eq:general_approx}
\end{equation}
Intuitively, (\ref{eq:general_approx}) is an expansion which allows the $N^{th}$ order interactions between the elements of the noise latent vector $\bm{z}$. Furthermore, (\ref{eq:general_approx}) resembles the functional form of a truncated Maclaurin expansion of vector-valued functions. In the case of a Maclaurin expansion, $\bmcal{W}^{[n]}$ represent the $n^{th}$ order partial derivatives of a known function. However, in our case the generator function is unknown and hence all the parameters need to be estimated from training samples.

The number of the unknown parameters in (\ref{eq:general_approx})
is $(d^{N+1} -1)\frac{o}{d-1}$, which grows  exponentially with the order of the approximation. Consequently,
the model of (\ref{eq:general_approx}) is  prone to overfitting and its training is computationally demanding.  

A natural approach to reduce the number of parameters is to assume that the weights exhibit redundancy and hence the parameter tensors are of low-rank.
To this end, several low-rank tensor decompositions can be employed \citep{kolda2009tensor,Sidiropoulos:16}. For instance,
let the parameter tensors $\bmcal{W}^{[n]}$ admit a CP decompostion \citep{kolda2009tensor} of mutilinear rank-$k$, namely, 
$\{ \bmcal{W}^{[n]} = [\![ \bm{U}_{[n], 1}, \bm{U}_{[n], 2}, \ldots, \bm{U}_{[n], (n+1)}  ]\!] \}_{n=1}^N$, with $\bm{U}_{[n], 1} \in  \realnum^{o\times k}$, and $\bm{U}_{[n], m} \in  \realnum^{d\times k}$, for $m=2,\dots,n+1$.
Then, (\ref{eq:general_approx}) is expressed as 
\begin{equation}
    G(\bm{z}) = \sum_{n=1}^N \bigg(
    [\![ \bm{U}_{[n], 1}, \bm{U}_{[n], 2}, \ldots, \bm{U}_{[n], (n+1)}  ]\!] \prod_{j=2}^{n+1} \times_{j} \bm{z}\bigg) + \bm{\beta},
    \label{eq:CP_approx}
\end{equation}
which has significantly less parameters than (\ref{eq:general_approx}), especially when $k \ll d$.
However, a set of different factor matrices for each level of approximation are required in \eqref{eq:CP_approx}, and hence
the hierarchical nature of images is not taken into account. 
To promote compositional structures and capture interactions among parameters in different orders of approximation we introduce next two coupled CP decompositions with shared factors.

\textbf{Model 1: \modelone{}}: 

Instead of factorizing each parameters tensor individually we propose to jointly factorize all the parameter tensors using a coupled CP decomposition with a specific pattern of factor sharing. To illustrate the factorization, we assume a third order approximation ($N=3$), however in the \supplementary{}
a generalization to $N$-th order approximation is provided. 
Let us assume that the parameters tensors admit the following coupled CP decomposition with the factors corresponding to lower-order levels of approximation being shared across all parameters tensors. That is: 
\begin{itemize}
    \item Let $\bm{W}^{[1]} = \bm{C}\bm{U}\matnot{1}^T$, be the parameters for first level of approximation.
    \item Let assume $\bmcal{W}^{[2]}$ being a superposition of of two weights tensors, namely $\bmcal{W}^{[2]} = \bmcal{W}^{[2]}_{1:2} + \bmcal{W}^{[2]}_{1:3}$, with $\bmcal{W}^{[2]}_{i:j}$ denoting parameters associated with the second order interactions across the $i$-th and $j$-th order of approximation. By enforcing the CP decomposition of the above tensors to share the factor with tensors corresponding to lower-order of approximation we obtain in matrix form: $\bm{W}^{[2]}_{(1)} = \bm{C}(\bm{U}\matnot{3} \odot \bm{U}\matnot{1})^T + \bm{C}(\bm{U}\matnot{2} \odot \bm{U}\matnot{1})^T$.
    \item Similarly, we enforce the third-order parameters tensor to admit the following CP decomposition (in matrix form) $\bm{W}^{[3]}_{(1)} = \bm{C}(\bm{U}\matnot{3} \odot \bm{U}\matnot{2} \odot \bm{U}\matnot{1})^T $. Note that all but the $\bm{U}\matnot{3}$ factor matrices are shared in the factorization of tensors capturing polynomial parameters for the first and second order of approximation.
\end{itemize}

The parameters are $\bm{C} \in \realnum^{o\times k}, \bm{U}\matnot{m} \in  \realnum^{d\times k}$ for $m=1,2,3$. Then, (\ref{eq:CP_approx}) for $N=3$ is written as:

\begin{equation}
\begin{split}
    G(\bm{z}) = \bm{\beta} + \bm{C}\bm{U}\matnot{1}^T\bm{z} + \bm{C}\Big(\bm{U}\matnot{3} \odot \bm{U}\matnot{1}\Big)^T(\bm{z} \odot \bm{z}) + \bm{C}\Big(\bm{U}\matnot{2} \odot \bm{U}\matnot{1}\Big)^T(\bm{z} \odot \bm{z}) + \\
    \bm{C}\Big(\bm{U}\matnot{3} \odot \bm{U}\matnot{2} \odot \bm{U}\matnot{1}\Big)^T(\bm{z} \odot \bm{z} \odot \bm{z})
\label{eq:polygan_recursive_gen_third_order}
\end{split}
\end{equation}

The third order approximation of (\ref{eq:polygan_recursive_gen_third_order}) can be implemented as a neural network with the structure of Figure~\ref{fig:polygan_model1_schematic} (proved in section~\ref{sec:polygan_theoretical_suppl}, Claim~\ref{lemma:polygan_additive_cp_model_hadamard_third_order} of the \supplementary). It is worth noting that the structure of the proposed network allows for incremental network growth.

\begin{figure*}[t]
    \centering
    \includegraphics[width=1\linewidth]{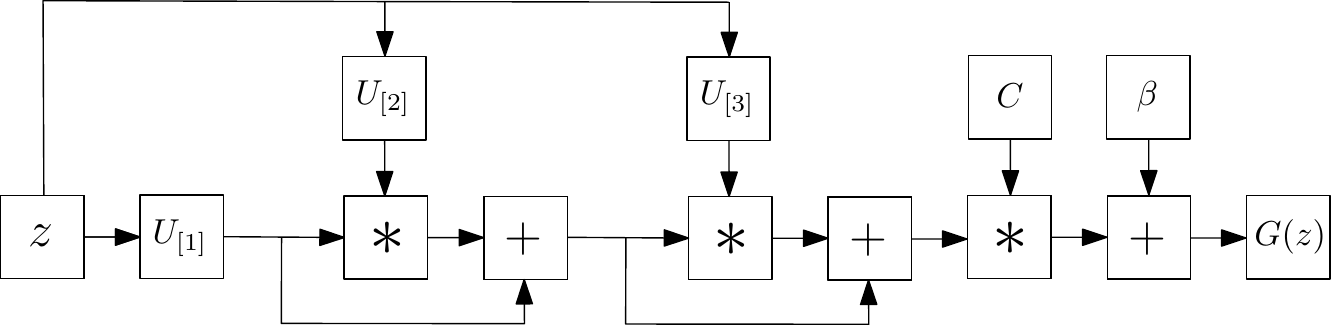}
\caption{Schematic illustration of the \modelone{} (for third order approximation). Symbol $*$ refers to the Hadamard product.}
\label{fig:polygan_model1_schematic}
\end{figure*}

\textbf{Model 2: \modeltwo{}}: Instead of explicitly separating the interactions between layers, we can utilize a joint hierarchical decomposition on the polynomial parameters. Let us first introduce learnable hyper-parameters $\big\{\bm{b}\matnot{n} \in \realnum^\omega\big\}_{n=1}^N$, which act as scaling factors for each parameter tensor. Therefore, we modify (\ref{eq:general_approx}) to:

\begin{equation}
    G(\bm{z}) = \sum_{n=1}^N \bigg(\bmcal{W}^{[n]} \times_2 \bm{b}\matnot{n} \prod_{j=3}^{n+2} \times_{j} \bm{z}\bigg) + \bm{\beta},
    \label{eq:polygan_general_polynomial_with_b}
\end{equation}

with $\big\{\bmcal{W}^{[n]} \in  \realnum^{o\times \omega \times \prod_{m=1}^{n}\times_m d}\big\}_{n=1}^N$. For illustration purposes, we consider a third order function approximation ($N=3$). That is,

\begin{equation}
    G(\bm{z}) = \bm{\beta} + \bmcal{W}^{[1]}\times_2\bm{b}\matnot{1}\times_3\bm{z} + \bmcal{W}^{[2]}\times_2\bm{b}\matnot{2}\times_3\bm{z}\times_4\bm{z} + 
    \bmcal{W}^{[3]}\times_2\bm{b}\matnot{3}\times_3\bm{z}\times_4\bm{z}\times_5\bm{z}
    \label{eq:polygan_third_order_init}
\end{equation}

To estimate its parameters we jointly factorize all parameters tensors by employing nested CP detecomposion with parameter sharing as follows (in matrix form)

\begin{itemize}
    \item First order parameters : $\bm{W}^{[1]}_{(1)} = \bm{C} (\bm{A}\matnot{3} \odot \bm{B}\matnot{3})^T$.
    \item Second order parametes: $\bm{W}^{[2]}_{(1)} = \bm{C} \bigg\{\bm{A}\matnot{3} \odot \bigg[\Big(\bm{A}\matnot{2} \odot \bm{B}\matnot{2}\Big) \bm{S}\matnot{3}\bigg]\bigg\}^T$.
    \item Third order parameters: $\bm{W}^{[3]}_{(1)} = \bm{C} \bigg\{\bm{A}\matnot{3} \odot \bigg[\bigg(\bm{A}\matnot{2} \odot \Big\{\Big(\bm{A}\matnot{1} \odot \bm{B}\matnot{1}\Big) \bm{S}\matnot{2}\Big\} \bigg)\bm{S}\matnot{3} \bigg]\bigg\}^T$
\end{itemize}

with $\bm{C} \in  \realnum^{o\times k}, \bm{A}\matnot{n} \in  \realnum^{d\times k}, \bm{S}\matnot{n} \in  \realnum^{k\times k}, \bm{B}\matnot{n} \in  \realnum^{\omega\times k}$ for $n=1,\ldots,N$.  
Altogether, (\ref{eq:polygan_third_order_init}) is written as:

\begin{equation}
\begin{split}
    G(\bm{z}) = \bm{\beta} + \bm{C} (\bm{A}\matnot{3} \odot \bm{B}\matnot{3})^T (\bm{z} \odot \bm{b}\matnot{3}) + 
    \bm{C} \bigg\{\bm{A}\matnot{3} \odot \bigg[\Big(\bm{A}\matnot{2} \odot \bm{B}\matnot{2}\Big) \bm{S}\matnot{3}\bigg]\bigg\}^T \Big(\bm{z} \odot \bm{z} \odot \bm{b}\matnot{2} \Big) + \\
    \bm{C} \bigg\{\bm{A}\matnot{3} \odot \bigg[\bigg(\bm{A}\matnot{2} \odot \Big\{\Big(\bm{A}\matnot{1} \odot \bm{B}\matnot{1}\Big) \bm{S}\matnot{2}\Big\} \bigg)\bm{S}\matnot{3} \bigg]\bigg\}^T \Big(\bm{z} \odot \bm{z} \odot \bm{z} \odot \bm{b}\matnot{1}\Big) 
    \label{eq:polygan_third_order_decomp_init}
\end{split}
\end{equation}

As we prove in the \supplementary{} (section~\ref{sec:polygan_theoretical_suppl}, Claim~\ref{lemma:polygan_lemma_third_order}), (\ref{eq:polygan_third_order_decomp_init}) can be implemented in a hierarchical manner with a three-layer neural network as shown in Figure~\ref{fig:polygan_core_module}.

\textbf{Comparison between the two models}: Both models are based on the polynomial expansion, however there are few differences between those. The \modelone{} has a simpler expression, however the \modeltwo{} relates to standard architectures using hierarchical composition that has recently yielded promising results in GANs (see Section~\ref{sec:polygan_related}). In the remainder of the paper, we use the \modeltwo{} by default; in Section~\ref{sec:polygan_suppl_model_comparison}, we include an experimental comparison of the two models. The experimental comparison demonstrates that neither model outperforms the other in all datasets; they perform similarly.

\begin{figure*}[!h]
    \centering
    \includegraphics[width=1\linewidth]{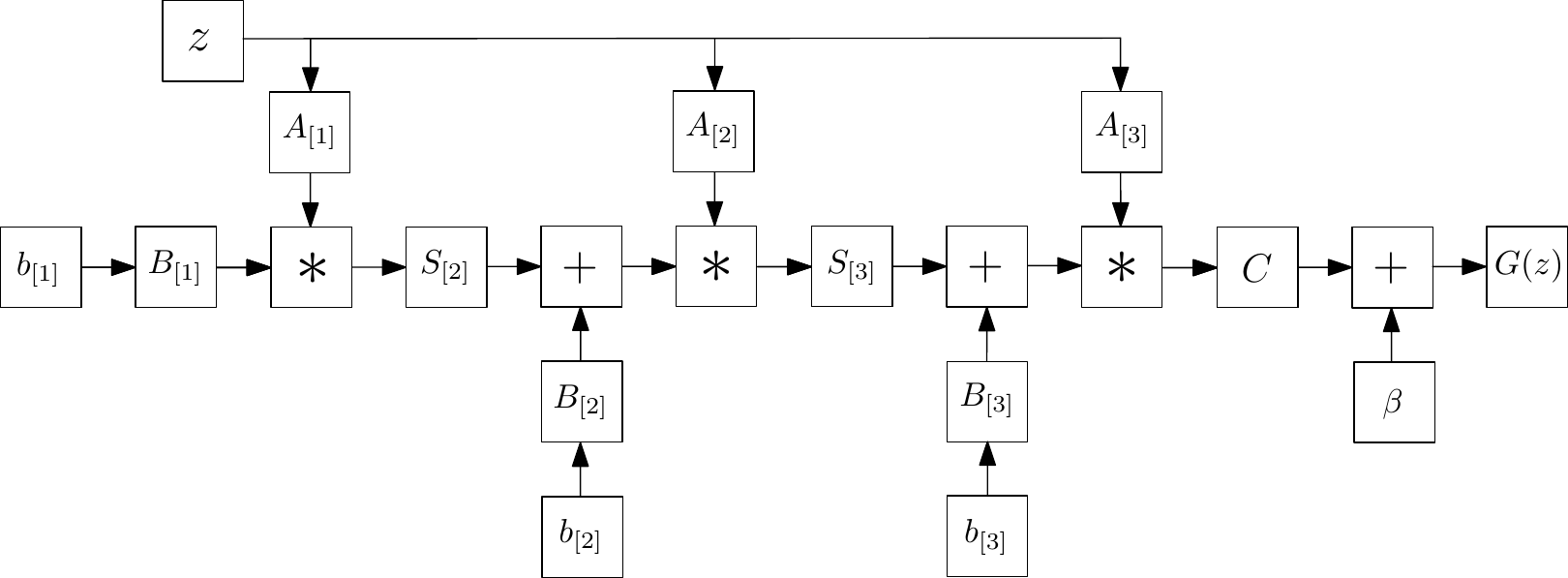}
\caption{Schematic illustration of the \modeltwo{} (for third order approximation). Symbol $*$ refers to the Hadamard product.}
\label{fig:polygan_core_module}
\end{figure*}

 \section{Related work}
\label{sec:polygan_related}
The literature on GANs is vast; we focus only on the works most closely related to ours. The interested reader can find further information in a recent survey~\citep{creswell2018generative}.

\citet{berthelot2017began} use skip connections to concatenate the noise $\bm{z}$ in deeper layers in the generator. The recent BigGAN \citep{brock2019large} performs a hierarchical composition through skip connections from the noise $\bm{z}$ to multiple resolutions of the generator. In their implementation, they split $\bm{z}$ into one chunk per resolution and concatenate each chunk (of $\bm{z}$) to the respective resolution.

Despite the propagation of the noise $\bm{z}$ to successive layers, the aforementioned works have substantial differences from ours. We introduce a well-motivated and mathematically elaborate method to achieve a more precise approximation with a polynomial expansion. In contrast to the previously mentioned works, we also do not concatenate the noise with the feature representations, but rather perform multiplication of the noise with the feature representations, which we mathematically justify. 

The work that is most closely related to ours is the recently proposed StyleGAN~\citep{karras2018style}, which is an improvement over the Progressive Growing of GANs (ProGAN)~\citep{karras2017progressive}. As ProGAN, StyleGAN is a highly-engineered network that achieves compelling results on synthesized 2D images. In order to provide an explanation on the improvements of StyleGAN over ProGAN, the authors adopt arguments from the style transfer literature \citep{huang2017arbitrary}. Nevertheless, the idea of style transfer proposes to use features from images for conditional image translation, which is very different to unsupervised samples (image) generation. We believe that these improvements can be better explained under the light of our proposed polynomial function approximation. That is, as we show in Figure~\ref{fig:powergan_skip_injections}, the Hadamard products build a hierachical decomposition with increasing level of detail (rather than different styles).  
In addition, the improvements in StyleGAN~\citep{karras2018style} are demonstrated by using a well-tuned model. In this paper we showcase that without any complicated engineering process the polynomial generation can be applied into several architectures (or any other type of decoders) and consistently improves the performance.

\section{Experiments}
\label{sec:polygan_experiments}
A sequence of experiments in both synthetic data ($2D$ and $3D$ data manifolds) and higher-dimensional signals are conducted to assess the empirical performance of the proposed polynomial expansion. The first experiments are conducted on a $2D$ manifolds that are analytically known (Section~\ref{sec:polygan_synthetic_2d}). Further experiments on three $3D$ manifolds are deferred to the \supplementary{} (Section~\ref{ssec:polygan_experiments_synthetic_suppl}). In Section~\ref{sec:polygan_linear_mnist}, the polynomial expansion is used for synthesizing digits. Experiments on images beyond digits are conducted in Section~\ref{sec:polygan_image_generation_linear_suppl}; more specifically, we experiment with images of faces and natural scenes. The experiments with such images demonstrate how polynomial expansion can be used for learning highly complex distributions by using a single activation function in the generator. Lastly, we augment our polynomial-based generator with non-linearities and show that this generator is at least as powerful as contemporary architectures.

Apart from the polynomial-based generators, we implemented two variations that are considered baselines: (a) `Concat': we replace the Hadamard operator with concatenation (used frequently in recent methods, such as in \citet{brock2019large}), (b) `Orig': the Hadamard products are ditched, while use $\bm{b}\matnot{1} \gets \bm{z}$, i.e., there is a composition of linear layers that transform the noise $\bm{z}$.

\subsection{Synthetic experiment on 2D manifold}
\label{sec:polygan_synthetic_2d}

\textbf{Sinusoidal}: We assess the polynomial-based generator on a sinusoidal function in the bounded domain $[0, 2\pi]$. Only linear blocks, i.e., no activation functions, are used in the generator. That is, all the element-wise non-linearities (such as ReLU's, $\tanh$) are ditched. The distribution we want to match is a $\sin{x}$ signal. The input to the generator is $z \in \realnum$ and the output is $[x, \sin{x}]$ with $x \in [0, 2\pi]$. We assume a $12^{th}$ order approximation where each $\bm{S}\matnot{i}, \bm{A}\matnot{i}$ is a fully-connected layer and $\bm{B}\matnot{i}$ is an identity matrix. Each fully-connected layer has width $15$. In Figure~\ref{fig:polygan_linear_sin}, $2,000$ random samples are synthesized. We indeed verify that in low-dimensional distributions, such as the univariate sinusoidal, \modelname{} indeed approximates the data distribution quite accurately without using any non-linear activation functions. 

\begin{figure}[h]
    \subfloat[GT]{\includegraphics[width=0.28\linewidth]{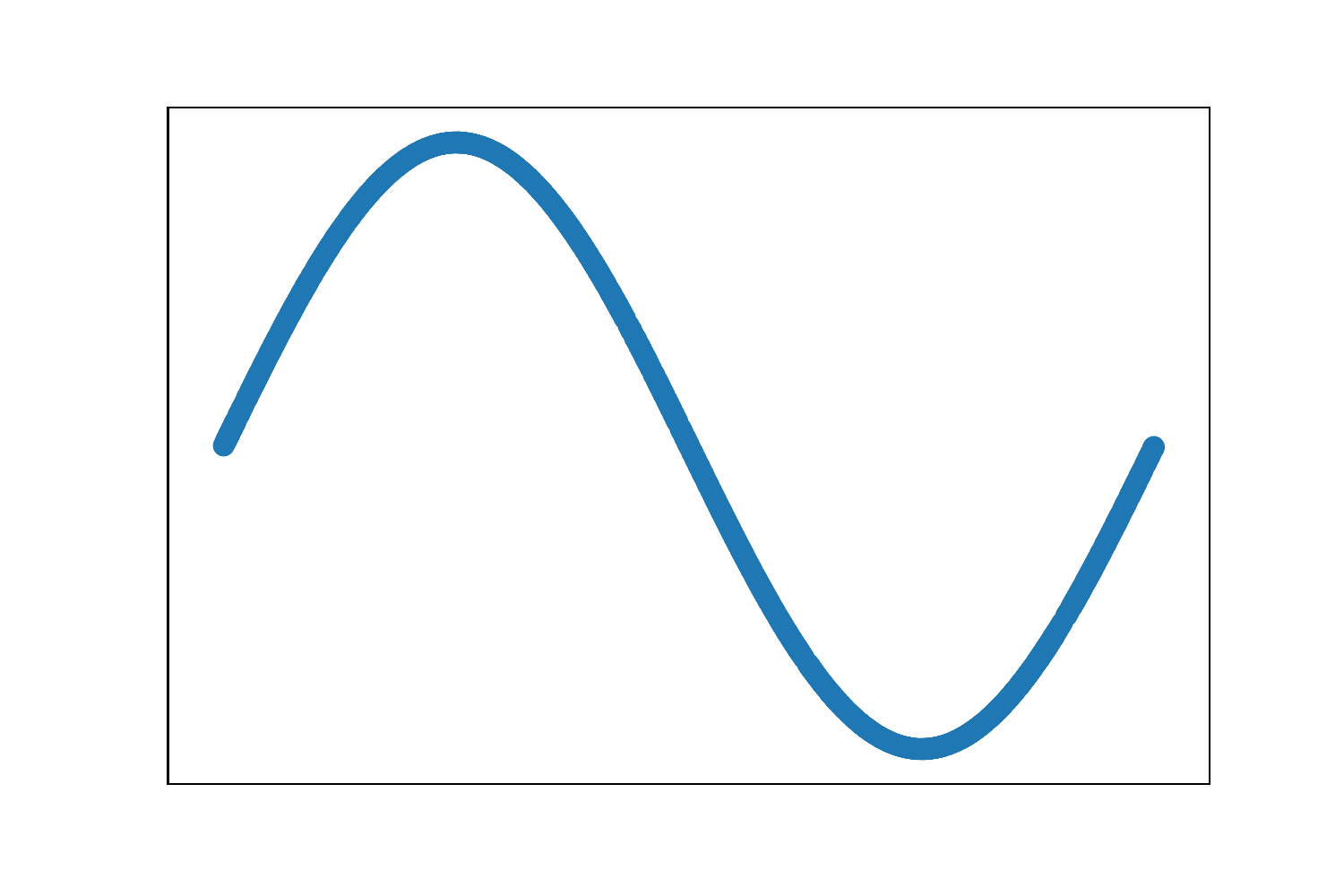}\hspace{-4mm}}
    \subfloat[Orig]{\includegraphics[width=0.28\linewidth]{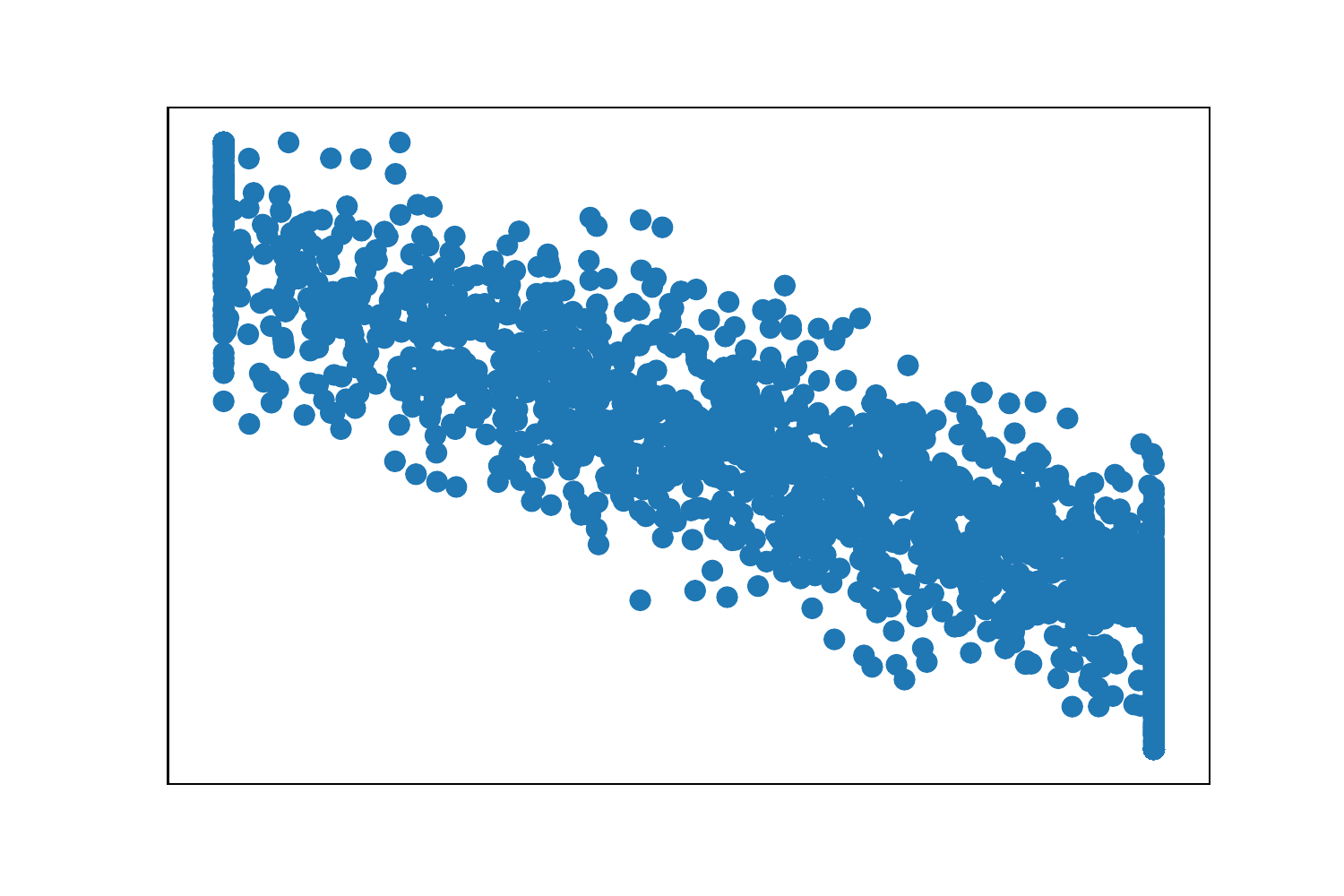}\hspace{-4mm}}
    \subfloat[Concat]{\includegraphics[width=0.28\linewidth]{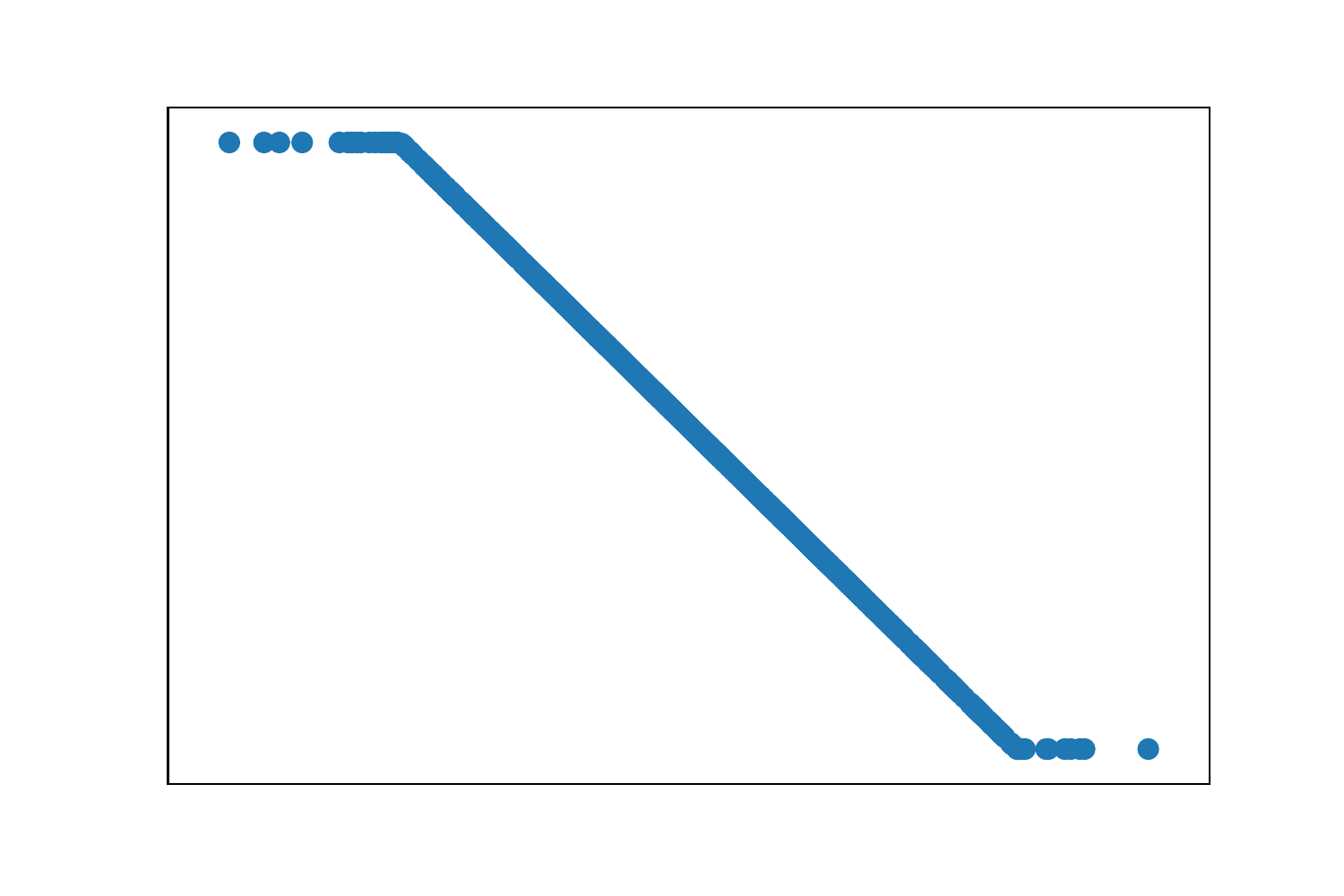}\hspace{-4mm}}
    \subfloat[\modelname]{\includegraphics[width=0.28\linewidth]{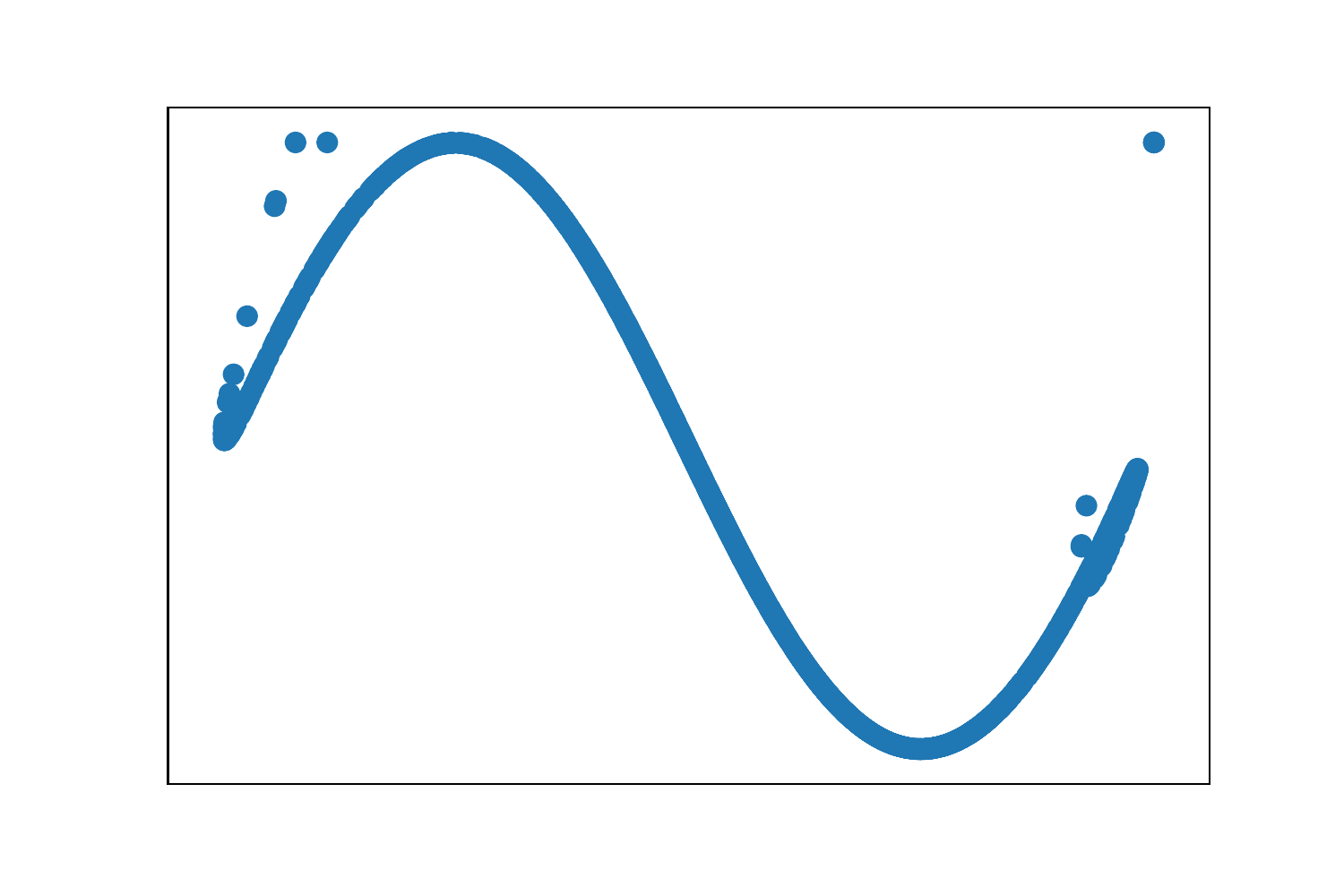}\hspace{-4mm}}
\caption{Synthesized data for learning the $[x, \sin{x}]$ signal. No activation functions are used in the generators. From left to right: (a) the data distribution, (b) `Orig', (c) `Concat', (d) \modelname. Notably, neither `Orig' nor `Concat' can learn to approximate different Taylor terms.}
\label{fig:polygan_linear_sin}
\end{figure}

\begin{figure}[htb]
    \subfloat[GT]{\includegraphics[width=0.24\linewidth]{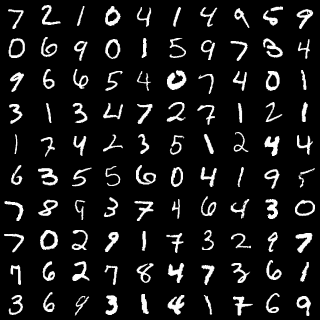}\hspace{1mm}}
    \subfloat[Orig]{\includegraphics[width=0.24\linewidth]{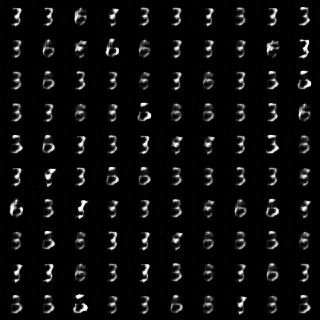}\hspace{1mm}}
    \subfloat[Concat]{\includegraphics[width=0.24\linewidth]{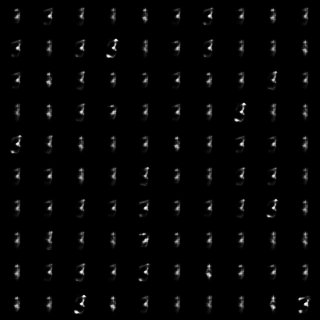}\hspace{1mm}}
    \subfloat[\modelname]{\includegraphics[width=0.24\linewidth]{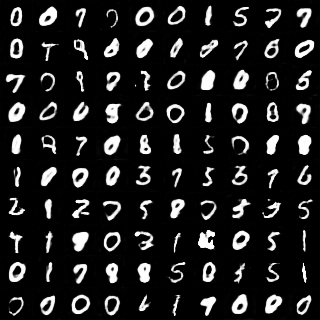}\hspace{1mm}}
\caption{Synthesized data for MNIST with a single activation in the generator. From left to right: (a) The ground-truth signals, (b) `Orig', (c) `Concat', (d) \modelname.}
\label{fig:polygan_linear_mnist_visual}
\end{figure}

\subsection{Digit generation with linear blocks}
\label{sec:polygan_linear_mnist}
The linear generator of the previous section is extended to greyscale images, in which an analytic expression of the ground-truth distribution remains elusive. To our knowledge, there has not been a generation of greyscale images based on polynomial expansion in the past. 

We capitalize on the expressivity of the recent resnet-based generator~\citep{miyato2018spectral, brock2019large}, to devise a new polynomial generator $G(\bm{z}) : \realnum^{128} \to \realnum^{32x32}$. We consider a fourth-order approximation (as derived in (\ref{eq:general_approx})) where $\bm{B}\matnot{i}$ is the identity matrix, $\bm{S}\matnot{i}$ is a residual block with two convolutions for $i=1,\dots,4$. We emphasize that the residual block as well as all layers are \emph{linear}, i.e., there are no activation functions. We only add a $\tanh$ in the output of the generator for normalization purposes. The discriminator and the optimization procedure are the same as in SNGAN; the only difference is that we run one discriminator step per generator step ($n_{dis} = 1$). Note that the `Orig' resnet-based generator resembles the generator of \citet{miyato2018spectral} in this case.

We perform digit generation (trained on MNIST~\citep{lecun1998gradient}). In Figure~\ref{fig:polygan_linear_mnist_visual}, random samples are visualized for the three compared methods. Note that the two baselines have practically collapsed into a single number each, whereas \modelname{} does synthesize plausible digits. 

To further assist the generation process, we utilize the labels and train a conditional GAN. That is, the class labels are used for conditional batch normalization. As illustrated in Figure~\ref{fig:polygan_linear_mnist_conditional_visual}, the samples synthesized are improved over the unsupervised setting. `Orig' and `Concat' still suffer from severe mode collapse, while \modelname{} synthesizes digits that have different thickness (e.g. $9$), style (e.g. $2$) and rotation (e.g. $1$).

\begin{figure}[htb]
    \subfloat[GT]{\includegraphics[width=0.24\linewidth]{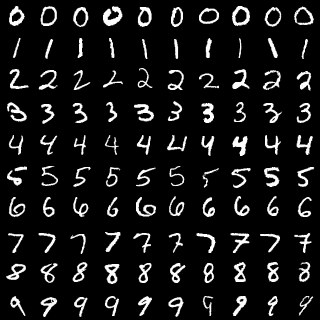}\hspace{1mm}}
    \subfloat[Orig]{\includegraphics[width=0.24\linewidth]{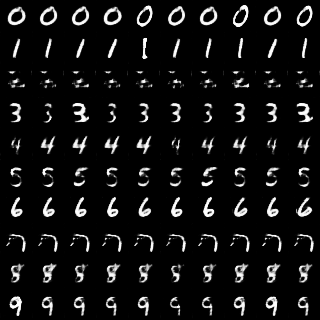}\hspace{1mm}}
    \subfloat[Concat]{\includegraphics[width=0.24\linewidth]{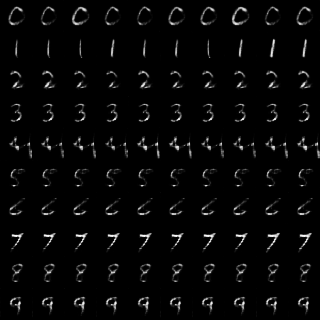}\hspace{1mm}}
    \subfloat[\modelname]{\includegraphics[width=0.24\linewidth]{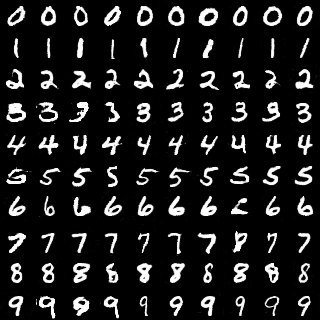}\hspace{1mm}}
\caption{Conditional digit generation. Note that both `Orig' and `Concat' suffer from severe mode collapse (details in section~\ref{sec:polygan_linear_mnist}). On the contrary, \modelname{} synthesizes digits that have different thickness (e.g. $9$), style (e.g. $2$) and rotation (e.g. $1$).}
\label{fig:polygan_linear_mnist_conditional_visual}
\end{figure} \section{Conclusion}
\label{sec:power_gan_conclusion}

We express data generation as a polynomial expansion task. We model the high-order polynomials with tensorial factors. We introduce two tailored coupled decompositions and show how the polynomial parameters can be implemented by hierarchical neural networks, e.g. as generators in a GAN setting. We exhibit how such polynomial-based generators can be used to synthesize images by utilizing only linear blocks. In addition, we empirically demonstrate that our polynomial expansion can be used with non-linear activation functions to improve the performance of standard state-of-the-art architectures. Finally, it is worth mentioning that our approach reveals links between high-order polynomials, coupled tensor decompositions and network architectures. 

\bibliographystyle{iclr2020_conference}
\bibliography{egbib}

\newpage
\appendix

\section{Introduction}
\label{sec:power_gan_introduction_suppl}

The appendix is organized as: 
\begin{itemize}
    \item Section~\ref{sec:polygan_theoretical_suppl} provides the Lemmas and their proofs required for our derivations. 
    \item Section~\ref{sec:polygan_method_suppl} generalizes the \modelone{} for $N^{th}$ order expansion.
    \item Section~\ref{sec:polygan_experiments_suppl} extends the experiments to 3D manifolds.
    \item In Section~\ref{sec:polygan_image_generation_linear_suppl}, additional experiments on image generation with linear blocks are conducted.
    \item Comparisons with popular GAN architectures are conducted in Section~\ref{sec:polygan_experiments_sota_networks}. Specifically, we utilize three popular generator architectures and devise their polynomial equivalent and perform comparisons on image generation. We also conduct an ablation study indicating how standard engineering techniques affect the image generation of the polynomial generator.
    \item In Section~\ref{sec:polygan_suppl_model_comparison}, a comparison between the two proposed decompositions is conducted on data distributions from the previous Sections.
\end{itemize}

\begin{table}
\begin{tabular}{l  l}
\parbox{0.48\textwidth}{
\begin{algorithm}[H]
	\begin{small}
	\SetKwInOut{Input}{Input}
	\SetKwInOut{Output}{Output}
	\SetKwInOut{Init}{Initialize}
	\Input{Noise $\bm{z} \in \realnum^{d}, N \in \naturalnum$}
	\Output{$\bm{x} \in \realnum^{o}$}

        \% global transformation(s) of $\bm{z}$.
        
        {$\bm{v}$ = Linear($\bm{z}$)}
        
        \% first layer.
        
        $\bm{\kappa} = (\bm{U}\matnot{1})^T \bm{v}$
        
        \For{n=2:N}{
            \% Perform the Hadamard product for the $n^{th}$ layer.
            
            $\bm{\kappa} = \bm{\kappa} + \Big((\bm{U}\matnot{n})^T \bm{v} \Big) * \bm{\kappa}$
            
        }
        $\bm{x} = \bm{\beta} + \bm{C} \bm{\kappa}$.
        \caption{\modelname{} (model 1).}
	\label{alg:polygan_generator_model1}
	\end{small}
\end{algorithm}

}
&
\parbox{0.48\textwidth}{
\begin{algorithm}[H]
	\begin{small}
	\SetKwInOut{Input}{Input}
	\SetKwInOut{Output}{Output}
	\SetKwInOut{Init}{Initialize}
	\Input{Noise $\bm{z} \in \realnum^{d}, N \in \naturalnum$}
	\Output{$\bm{x} \in \realnum^{o}$}

        \% global transformation(s) of $\bm{z}$.
        
        {$\bm{v}$ = Linear($\bm{z}$)}
        
        \% first layer.
        
        $\bm{\kappa} = \Big((\bm{B}\matnot{1})^T \bm{b}\matnot{1} \Big) * \Big((\bm{A}\matnot{1})^T \bm{v}\Big)$
        
        \For{n=2:N}{
            \% Multiply with the current layer weight $\bm{S}\matnot{n}$ and perform the Hadamard product.
            $\bm{\kappa} = \Big(\bm{S}\matnot{n} \bm{\kappa} + (\bm{B}\matnot{n})^T \bm{b}\matnot{n} \Big) * \Big((\bm{A}\matnot{n})^T \bm{v}\Big)$
            
        }
        $\bm{x} = \bm{\beta} + \bm{C} \bm{\kappa}$.
        \caption{\modelname{} (model 2).}
	\label{alg:polygan_generator_model2}
	\end{small}
\end{algorithm}
}
\end{tabular}
\caption{The pseudocode for the two models for $N^{th}$ order polynomial approximation.}
\end{table} 

 \section{Proofs}
\label{sec:polygan_theoretical_suppl}

For a set of matrices  $\{\bm{X}_m \in \mathbb{R}^{I_m \times N} \}_{m=1}^N$ the Khatri-Rao product  is denoted by:
\begin{equation}
    \bm{X}_1 \odot \bm{X}_2 \odot  \cdots \odot  \bm{X}_M \doteq  \bigodot_{m=1}^M \bm{X}_m
\end{equation}

In this section, we prove the following identity connecting the sets of matrices $\{\bm{A}_{\nu} \mathbb{R}^{I_{\nu} \times K} \}_{\nu=1}^N$  and $\{\bm{B}_{\nu} \mathbb{R}^{I_{\nu} \times L} \}_{\nu=1}^N$: 

\begin{equation}
    (\bigodot_{\nu=1}^N \bm{A}_{\nu})^T \cdot (\bigodot_{\nu=1}^N \bm{B}_{\nu}) = (\bm{A}_1^T \cdot \bm{B}_1) * (\bm{A}_2^T \cdot \bm{B}_2) * \ldots * (\bm{A}_N^T \cdot \bm{B}_N)
\end{equation}

To demonstrate the simple case with two matrices, we prove first the special case with $N=2$.
\begin{lemma}
It holds that
\begin{equation}
(\bm{A}_1 \odot \bm{A}_2)^T \cdot (\bm{B}_1 \odot \bm{B}_2) = (\bm{A}_1^T \cdot \bm{B}_1) * (\bm{A}_2^T \cdot \bm{B}_2)
\label{eq:power_gan_suppl_lemma1_n2}
\end{equation}
\label{lemma:power_gan_lemma_hadamard_kr}
\end{lemma}

\begin{proof}
Initially, both sides of the equation have dimensions of $K\times L$, i.e., they match. 
The $(i,j)$ element of the matrix product of $(\bm{A}_1^T \cdot \bm{B}_1)$ is 

\begin{equation}
    \sum_{k_1=1}^{I_1} A_{1, (k_1, i)} B_{1, (k_1, j)}
\end{equation}

Then the $(i, j)$ element of the right hand side (rhs) of (\ref{eq:power_gan_suppl_lemma1_n2}) is:
\begin{equation}
\begin{split}
E_{rhs} = (\sum_{k_1=1}^{I_1} A_{1, (k_1, i)} B_{1, (k_1, j)}) \cdot (\sum_{k_2=1}^{I_2} A_{2, (k_2, i)} B_{2, (k_2, j)}) = \\
\sum_{k_1=1}^{I_1} \sum_{k_2=1}^{I_2} (A_{1, (k_1, i)} A_{2, (k_2, i)}) ( B_{(1, k_1, j)}  B_{2, (k_2, j)} )
\end{split}
\end{equation}
From the definition of Khatri-Rao, it is straightforward to obtain the $(\rho, i)$ element with $\rho = (k_1-1) I_2 + k_2$, (i.e. $\rho \in [1, I_1 I_2]$) of $\bm{A}_1 \odot \bm{A}_2$ as $A_{1, (k_1, i)} A_{2, (k_2, i)}$. Similarly, the $(\rho, j)$ element of $\bm{B}_1 \odot \bm{B}_2$ is $B_{1, (k_1, j)} B_{2, (k_2, j)}$.

The respective $(i, j)$ element of the left hand side (lhs) of (\ref{eq:power_gan_suppl_lemma1_n2}) is:
\begin{equation}
\begin{split}
    E_{lhs} = \sum_{\rho=1}^{I_1 I_2} A_{1, (k_1, i)} A_{2, (k_2, i)} B_{1, (k_1, j)} B_{2, (k_2, j)} = \\
    \sum_{k_1=1}^{I_1} \sum_{k_2=1}^{I_2} A_{1, (k_1, i)} A_{2, (k_2, i)} B_{1, (k_1, j)} B_{2, (k_2, j)} = E_{rhs}
\end{split}
\end{equation}

In the last equation, we replace the sum in $\rho$ ($\rho \in [1, I_1 I_2]$) with the equivalent sums in $k_1, k_2$.
\end{proof}

In a similar manner, we generalize the identity to the case of $N>2$ terms below.

\begin{lemma}
It holds that
\begin{equation}
    (\bigodot_{\nu=1}^N \bm{A}_{\nu})^T \cdot (\bigodot_{\nu=1}^N \bm{B}_{\nu}) = (\bm{A}_1^T \cdot \bm{B}_1) * (\bm{A}_2^T \cdot \bm{B}_2) * \ldots * (\bm{A}_N^T \cdot \bm{B}_N)
\label{eq:polygan_suppl_lemma1_N}
\end{equation}
\label{lemma:polygan_lemma_hadamard_kr2}
\end{lemma}

\begin{proof}
The rhs includes the Hadamard products of the matrices $\bm{A}_{\nu}^T \cdot \bm{B}_{\nu}$. Each matrix multiplication ($\bm{A}_{\nu}^T \cdot \bm{B}_{\nu}$) results in a matrix of $K \times L$ dimensions. Thus, the rhs is a matrix of $K \times L$ dimensions. 

The lhs is a matrix multiplication of two Khatri-Rao products. The first Khatri-Rao product has dimensions $K \times (\prod_{\nu} I_{\nu})$, while the second $(\prod_{\nu} I_{\nu}) \times L$. Altogether, the lhs has $K \times L$ dimensions.

Similarly to the previous Lemma, the $(i, j)$ element of the rhs is:
\begin{equation}
\begin{split}
    E_{rhs} = (\sum_{k_1=1}^{I_1} A_{1, (k_1, i)} B_{1, (k_1, j)}) \cdot (\sum_{k_2=1}^{I_2} A_{2, (k_2, i)} B_{2, (k_2, j)}) \ldots   (\sum_{k_N=1}^{I_N} A_{N, (k_N, i)} B_{N, (k_N, j)}) = \\
    \sum_{k_1=1}^{I_1} \sum_{k_2=1}^{I_2} \ldots \sum_{k_N=1}^{I_N} (A_{1, (k_1, i)} A_{2, (k_2, i)} \ldots A_{N, (k_N, i)}) ( B_{1, (k_1, j)}  B_{2, (k_2, j)} \ldots B_{N, (k_N, j)})
\end{split}
\end{equation}

To proceed with the lhs, it is straightforward to derive that 
\begin{equation}
    (\bigodot_{\nu=1}^N \bm{A}_{\nu}) = A_{1, (s_1, i)} A_{2, (s_2, i)} \ldots A_{N, (s_N, i)}
\end{equation}

where $s_1 = i$ and $s_{\nu}$ is a recursive function of the $s_{\nu-1}$. 

However, the recursive definition of $s_{\nu}$ is summed in the multiplication and we obtain: 
\begin{equation}
    E_{lhs} = \sum_{k_1=1}^{I_1} \sum_{k_2=1}^{I_2} \ldots \sum_{k_N=1}^{I_N} (A_{1, (k_1, i)} A_{2, (k_2, i)} \ldots A_{N, (k_N, i)}) ( B_{1, (k_1, j)}  B_{2, (k_2, j)} \ldots B_{N, (k_N, j)}) = E_{rhs}
\end{equation}
\end{proof}

\subsection{Proofs for model 1}
\label{ssec:polygan_proofs_lemma_additive_cp_model}

Below, we prove that (\ref{eq:polygan_recursive_gen_third_order}) (main paper) is equivalent to the three-layer neural network as shown in Figure~\ref{fig:polygan_model1_schematic}.

\begin{claim}
Let $\bm{\omega} = \Big(\bm{U}\matnot{2}^T \bm{z} \Big) * \Big(\bm{U}\matnot{1}^T \bm{z} \Big) + \bm{U}\matnot{1}^T \bm{z}$.

    Then, the form of (\ref{eq:polygan_recursive_gen_third_order}) is equal to:
    \begin{equation}
        G(\bm{z}) = \bm{\beta} + \bm{C}\bigg\{\Big(\bm{U}\matnot{3}^T \bm{z}\Big) * \bm{\omega} + \bm{\omega}\bigg\}
        \label{eq:polygan_lemma_additive_cp_model_hadamard_form1}
    \end{equation}
  \label{lemma:polygan_additive_cp_model_hadamard_third_order}
\end{claim}

\begin{proof}
Applying Lemma~\ref{lemma:polygan_lemma_hadamard_kr2} on (\ref{eq:polygan_recursive_gen_third_order}), we obtain:
\begin{equation}
\begin{split}
    G(\bm{z}) = \bm{\beta} + \bm{C}\bigg\{ \bm{U}\matnot{1}^T\bm{z} + \Big(\bm{U}\matnot{3}^T \bm{z}\Big) * \Big(\bm{U}\matnot{1}^T \bm{z}\Big) + \Big(\bm{U}\matnot{2}^T \bm{z} \Big) * \Big( \bm{U}\matnot{1}^T \bm{z} \Big) + \\
    \Big(\bm{U}\matnot{3}^T \bm{z}\Big) * \Big(\bm{U}\matnot{2}^T \bm{z} \Big) * \Big(\bm{U}\matnot{1}^T \bm{z}\Big) \bigg\} = \\
    \bm{\beta} + \bm{C}\bigg\{ \Big(\bm{U}\matnot{3}^T \bm{z}\Big) * \Big[\Big(\bm{U}\matnot{2}^T \bm{z} \Big) * \Big(\bm{U}\matnot{1}^T \bm{z} \Big) + \bm{U}\matnot{1}^T \bm{z} \Big] + \Big(\bm{U}\matnot{2}^T \bm{z} \Big) * \Big(\bm{U}\matnot{1}^T \bm{z} \Big) + \bm{U}\matnot{1}^T \bm{z}\bigg\}
\end{split}
\end{equation}

The last equation is the same as (\ref{eq:polygan_lemma_additive_cp_model_hadamard_form1}).
\end{proof}

\subsection{Proofs for model 2}
\label{ssec:polygan_proofs_lemma_model_hierarchical}

In Claim~\ref{lemma:polygan_lemma_third_order_x2_term} and Claim~\ref{lemma:polygan_lemma_third_order}, we prove that (\ref{eq:polygan_third_order_decomp_init}) (main paper) is equivalent to the three-layer neural network as shown in Figure~\ref{fig:polygan_core_module}.

\begin{claim}
Let 
\begin{equation}
    \bm{\omega} = \bigg((\bm{A}\matnot{2})^T \bm{z} \bigg) * \bigg\{ (\bm{B}\matnot{2})^T \bm{b}\matnot{2} + (\bm{S}\matnot{2})^T \bigg[ \bigg( (\bm{A}\matnot{1})^T \bm{z} \bigg) * \bigg( (\bm{B}\matnot{1})^T \bm{b}\matnot{1} \bigg) \bigg] \bigg\}
    \label{eq:polygan_lemma_x2_hadamard_form}
\end{equation}

It holds that 
\begin{equation}
    \bm{\omega} = \bigg\{\bm{A}\matnot{2} \odot \Big[\Big(\bm{A}\matnot{1} \odot \bm{B}\matnot{1}\Big) \bm{S}\matnot{2}\Big] \bigg\}^T \Big(\bm{z} \odot \bm{z} \odot \bm{b}\matnot{1}\Big) + \Big(\bm{A}\matnot{2} \odot \bm{B}\matnot{2}\Big)^T \Big(\bm{z} \odot \bm{b}\matnot{2} \Big)
    \label{eq:polygan_lemma_x2_kr_form}
\end{equation}
\label{lemma:polygan_lemma_third_order_x2_term}
\end{claim}

\begin{proof}
We will prove the equivalence starting from (\ref{eq:polygan_lemma_x2_hadamard_form}) and transform it into (\ref{eq:polygan_lemma_x2_kr_form}). From (\ref{eq:polygan_lemma_x2_hadamard_form}):

\begin{equation}
\begin{split}
    \bm{\omega} =  \Big((\bm{A}\matnot{2})^T \bm{z} \Big) * \Big((\bm{B}\matnot{2})^T \bm{b}\matnot{2}\Big) + \Big((\bm{A}\matnot{2})^T \bm{z} \Big) * \bigg\{ (\bm{S}\matnot{2})^T \bigg[ \bigg( (\bm{A}\matnot{1})^T \bm{z} \bigg) * \bigg( (\bm{B}\matnot{1})^T \bm{b}\matnot{1} \bigg) \bigg] \bigg\} = \\
    \Big(\bm{A}\matnot{2} \odot \bm{B}\matnot{2}\Big)^T \Big(\bm{z} \odot \bm{b}\matnot{2} \Big) + \Big((\bm{A}\matnot{2})^T \bm{z} \Big) * \bigg\{ \bigg[ \bigg( \bm{A}\matnot{1} \odot  \bm{B}\matnot{1}\bigg) \bm{S}\matnot{2}  \bigg]^T \Big(\bm{z} \odot \bm{b}\matnot{1}\Big) \bigg\}
\end{split}
\label{eq:polygan_lemma_x2_proof_part1}
\end{equation}

where in the last equation, we have applied Lemma~\ref{lemma:power_gan_lemma_hadamard_kr}. Applying the Lemma once more in the last term of (\ref{eq:polygan_lemma_x2_proof_part1}), we obtain (\ref{eq:polygan_lemma_x2_kr_form}).
\end{proof}

\begin{claim}
Let
\begin{equation}
\bm{\lambda} = \bm{\beta} + \bm{C} \bigg\{ \bigg((\bm{A}\matnot{3})^T \bm{z}\bigg) * \bigg[ (\bm{B}\matnot{3})^T \bm{b}\matnot{3} + (\bm{S}\matnot{3})^T \bm{\omega} \bigg] \bigg\}
\label{eq:polygan_lemma_x3_hadamard_form}
\end{equation}

with $\bm{\omega}$ as in Claim~\ref{lemma:polygan_lemma_third_order_x2_term}.  Then, it holds for $G(\bm{z})$ of (\ref{eq:polygan_third_order_decomp_init}) that $G(\bm{z}) = \bm{\lambda}$.

\label{lemma:polygan_lemma_third_order}
\end{claim}

\begin{proof}
Transforming (\ref{eq:polygan_lemma_x3_hadamard_form}) into (\ref{eq:polygan_third_order_decomp_init}):

\begin{equation}
\begin{split}
    \bm{\lambda} = \bm{\beta} + \bm{C} \bigg\{ \bigg((\bm{A}\matnot{3})^T \bm{z}\bigg) * \bigg( (\bm{B}\matnot{3})^T \bm{b}\matnot{3} \bigg) + \bigg((\bm{A}\matnot{3})^T \bm{z}\bigg) * \bigg( (\bm{S}\matnot{3})^T \bm{\omega} \bigg) \bigg\} = \\
    \bm{\beta} + \bm{C} \bigg\{  \bigg(\bm{A}\matnot{3} \odot \bm{B}\matnot{3} \bigg)^T \bigg(\bm{z} \odot \bm{b}\matnot{3}\bigg) + \bigg((\bm{A}\matnot{3})^T \bm{z}\bigg) * \bigg( (\bm{S}\matnot{3})^T \bm{\omega} \bigg) \bigg\}
\end{split}
\label{eq:polygan_lemma_x3_proof_part1}
\end{equation}

To simplify the notation, we define $\bm{M}_1 = \bigg\{\bm{A}\matnot{2} \odot \Big[\Big(\bm{A}\matnot{1} \odot \bm{B}\matnot{1}\Big) \bm{S}\matnot{2}\Big] \bigg\}$ and $\bm{M}_2 = \Big(\bm{A}\matnot{2} \odot \bm{B}\matnot{2}\Big)$. Then, $\bm{\omega} = \bm{M}_1^T \Big(\bm{z} \odot \bm{z} \odot \bm{b}\matnot{1}\Big) + \bm{M}_2^T \Big(\bm{z} \odot \bm{b}\matnot{2} \Big)$. The last term of (\ref{eq:polygan_lemma_x3_proof_part1}) becomes:

\begin{equation}
\begin{split}
    \bigg((\bm{A}\matnot{3})^T \bm{z}\bigg) * \bigg( (\bm{S}\matnot{3})^T \bm{\omega} \bigg) = \bigg((\bm{A}\matnot{3})^T \bm{z}\bigg) * \bigg[ \bigg( \bm{M}_1 \bm{S}\matnot{3}\bigg)^T \Big(\bm{z} \odot \bm{z} \odot \bm{b}\matnot{1}\Big) + \\
    \bigg(\bm{M}_2 \bm{S}\matnot{3} \bigg)^T \Big(\bm{z} \odot \bm{b}\matnot{2} \Big) \bigg] = \\
    \bigg[ \bm{A}\matnot{3} \odot \bigg(\bm{M}_1 \bm{S}\matnot{3} \bigg) \bigg]^T \Big(\bm{z} \odot \bm{z} \odot \bm{z} \odot \bm{b}\matnot{1}\Big) + \bigg[ \bm{A}\matnot{3} \odot \bigg(\bm{M}_2 \bm{S}\matnot{3} \bigg) \bigg]^T \Big(\bm{z} \odot \bm{z} \odot \bm{b}\matnot{2}\Big)
\end{split}
\label{eq:polygan_lemma_x3_proof_part2}
\end{equation}

Replacing (\ref{eq:polygan_lemma_x3_proof_part2}) into (\ref{eq:polygan_lemma_x3_proof_part1}), we obtain (\ref{eq:polygan_third_order_decomp_init}).
\end{proof}

Note that the $\bm{\lambda}$ in Claim~\ref{lemma:polygan_lemma_third_order} is the equation behind Figure~\ref{fig:polygan_core_module}. By proving the claim, we have illustrated how the polynomial generator can be transformed into a network architecture for third-order approximation.

 \section{Derivations}
\label{sec:polygan_method_suppl}

In this Section, we will show how the \modelone{} generalizes to the $N^{th}$ order approximation. It suffices to find the decomposition that converts the $N^{th}$ order polynomial into a network structure (see Alg.~\ref{alg:polygan_generator_model1}). 

As done in Section~\ref{sec:polygan_decomposition}, we capture the $n^{th}$ order interactions by decomposing the parameter tensor $\bmcal{W}^{[n]}$ (with $2 \leqslant n \leqslant N$) as:
\begin{equation}
    \bmcal{W}^{[n]} = \underbrace{\sum_{j_1=n}^N \sum_{j_2=n-1}^{j_1-1}\ldots \sum_{j_{n-1}=2}^{j_{n-2}-1}}_{(n-1) \text{sums}} \bmcal{W}^{[n]}_{1:j_{n-1}:\ldots:j_1}
    \label{eq:polygan_nth_order_interactions_coupled_decomposition}
\end{equation}

The term $\bmcal{W}^{[n]}_{1:j_{n-1}:\ldots:j_1}$ denotes the interactions across the layers $1, j_{n-1}, \ldots, j_1$. The $N^{th}$ order approximation becomes:

\begin{equation}
    G(\bm{z}) = \bm{\beta} + \sum_{n=1}^N \bigg\{\bigg(\sum_{j_1=n}^N \sum_{j_2=n-1}^{j_1-1}\ldots \sum_{j_{n-1}=2}^{j_{n-2}-1} \Big(\bmcal{W}^{[n]}_{1:j_{n-1}:\ldots:j_1}\Big)_{(1)}\bigg) \bigg(\bigodot_{m = 1}^{n} \bm{z} \bigg) \bigg\}
    \label{eq:polygan_nth_order_approximation_tensor_interactions}
\end{equation}

By considering the mode-1 unfoding of \modelone{} (like in Section~\ref{sec:polygan_decomposition}), we obtain:

\begin{equation}
\begin{split}
    G(\bm{z}) = \bm{\beta} +  \bm{C} \sum_{n=1}^N \bigg\{\sum_{j_1=n}^N \sum_{j_2=n-1}^{j_1-1}\ldots \sum_{j_{n-1}=2}^{j_{n-2}-1} \Big(\bm{U}\matnot{j_1} \odot \ldots \odot \bm{U}\matnot{j_{n-1}} \odot \bm{U}\matnot{1}\Big)^T \Big(\bigodot_{m = 1}^{n} \bm{z} \Big) \bigg\}  = \\
    \bm{\beta} +  \bm{C} \sum_{n=1}^N \bigg\{\sum_{j_1=n}^N \sum_{j_2=n-1}^{j_1-1}\ldots \sum_{j_{n-1}=2}^{j_{n-2}-1} \bigg( \Big(\bm{U}\matnot{j_1}^T \bm{z}\Big) * \ldots * \Big(\bm{U}\matnot{j_{n-1}}^T \bm{z}\Big) * \Big(\bm{U}\matnot{1}^T \bm{z}\Big) \bigg) \bigg\} = \bm{\beta} +  \bm{C} \bm{x}_{N}
\end{split}
\label{eq:polygan_nth_order_approximation}
\end{equation}

where we use $\bm{x}_{N}$ as an abbreviation of the sums. In the last equation, we have used Lemma~\ref{lemma:polygan_lemma_hadamard_kr2} (Section~\ref{sec:polygan_theoretical_suppl}).

\begin{claim}
The $N^{th}$ order approximation of (\ref{eq:polygan_nth_order_approximation_tensor_interactions}) can be implemented with a neural network as described in Alg.~\ref{alg:polygan_generator_model1}. 
\label{lemma:polygan_nth_order_approximation_to_network}
\end{claim}

\begin{proof}
We will use induction to prove the  Claim. 
For $N=2$, it trivially holds, while the proof for $N=3$ is provided in Claim~\ref{lemma:polygan_additive_cp_model_hadamard_third_order}. Suppose it holds for $N^{th}$ order approximation; we prove below that it holds for ${N+1}^{th}$ order approximation.

Let us denote the approximation of (\ref{eq:polygan_nth_order_approximation_tensor_interactions}) as $G_{N}(\bm{z})$.
The ${(N+1)}^{th}$ order approximation from (\ref{eq:polygan_nth_order_approximation_tensor_interactions}) is:

\begin{equation}
\begin{split}
    G_{N+1}(\bm{z}) = \bm{\beta} + \sum_{n=1}^{N+1} \bigg\{\bigg(\sum_{j_1=n}^{N+1} \sum_{j_2=n-1}^{j_1-1}\ldots \sum_{j_{n-1}=2}^{j_{n-2}-1} \Big(\bmcal{W}^{[n]}_{1:j_{n-1}:\ldots:j_1}\Big)_{(1)}\bigg) \bigg(\bigodot_{m = 1}^{n} \bm{z} \bigg) \bigg\} = \\
    \bm{\beta} + \sum_{n=1}^N \bigg\{\bigg(\sum_{j_1=n}^N \sum_{j_2=n-1}^{j_1-1}\ldots \sum_{j_{n-1}=2}^{j_{n-2}-1} \Big(\bmcal{W}^{[n]}_{1:j_{n-1}:\ldots:j_1}\Big)_{(1)}\bigg) \bigg(\bigodot_{m = 1}^{n} \bm{z} \bigg)  + \\
    \sum_{j_2=n-1}^{N}\ldots \sum_{j_{n-1}=2}^{j_{n-2}-1} \Big(\bmcal{W}^{[n]}_{1:j_{n-1}:\ldots:j_2:(N+1)}\Big)_{(1)} \Big(\bigodot_{m = 1}^{n} \bm{z} \Big) \bigg\} + \Big(\bmcal{W}^{[N+1]}_{1:2:\ldots:(N+1)}\Big)_{(1)} \Big(\bigodot_{m = 1}^{N+1} \bm{z} \Big) 
\end{split}
\label{eq:polygan_nplusone_order_approximation_tensor_interactions_induction}
\end{equation}

In the last equation, the first term in the sums is $\bm{x}_{N}$; for the rest two terms we apply Lemma~\ref{lemma:polygan_lemma_hadamard_kr2}:

\begin{equation}
\begin{split}
    G_{N+1}(\bm{z}) = \bm{\beta} + \bm{C}\bm{x}_N + \bm{C}\bigg\{\Big(\bm{U}\matnot{N+1}^T \bm{z}\Big) * \Big(\bm{U}\matnot{N}^T \bm{z}\Big) * \ldots * \Big(\bm{U}\matnot{2}^T \bm{z}\Big) * \Big(\bm{U}\matnot{1}^T \bm{z}\Big)\bigg\} + \\
    \bm{C} \sum_{n=1}^N \bigg\{\sum_{j_2=n-1}^{N}\ldots \sum_{j_{n-1}=2}^{j_{n-2}-1} \bigg( \Big(\bm{U}\matnot{N+1}^T \bm{z}\Big) * \Big(\bm{U}\matnot{j_2}^T \bm{z}\Big) * \ldots * \Big(\bm{U}\matnot{j_{n-1}}^T \bm{z}\Big) * \Big(\bm{U}\matnot{1}^T \bm{z}\Big) \bigg) \bigg\} = \\
    \bm{\beta} + \bm{C}\bm{x}_N + \bm{C} \bigg\{ \Big(\bm{U}\matnot{N+1}^T \bm{z}\Big) * \bigg[ \Big(\bm{U}\matnot{N}^T \bm{z}\Big) * \ldots * \Big(\bm{U}\matnot{2}^T \bm{z}\Big) * \Big(\bm{U}\matnot{1}^T \bm{z}\Big) + \bm{\lambda} \bigg]\bigg\}
\end{split}
\label{eq:polygan_nplusone_order_approximation_tensor_interactions_induction_part2}
\end{equation}

where 
\begin{equation}
    \bm{\lambda} = \sum_{n=1}^N \sum_{j_2=n-1}^{N}\ldots \sum_{j_{n-1}=2}^{j_{n-2}-1} \bigg( \Big(\bm{U}\matnot{j_2}^T \bm{z}\Big) * \ldots * \Big(\bm{U}\matnot{j_{n-1}}^T \bm{z}\Big) * \Big(\bm{U}\matnot{1}^T \bm{z}\Big) \bigg)
\end{equation}
The term $\bm{\lambda}$ is equal to the $\kappa=(n-1)^{th}$ order of (\ref{eq:polygan_nth_order_approximation}), while there is only a single term for $n=N$. Therefore, (\ref{eq:polygan_nplusone_order_approximation_tensor_interactions_induction_part2}) is transformed into:

\begin{equation}
    G_{N+1}(\bm{z}) = \bm{\beta} + \bm{C}\bm{x}_N + \bm{C} \Big\{ \Big(\bm{U}\matnot{N+1}^T \bm{z}\Big) * \bm{x}_N \Big\}
\end{equation}

which is exactly the form described by Alg.~\ref{alg:polygan_generator_model1}. This concludes the induction proof.

\end{proof}

\section{Experiments on surfaces}
\label{sec:polygan_experiments_suppl}

\textbf{Astroid}: We implement a superellipse with parametric expression $[\alpha \cos^3{t}, \alpha \sin^3{t}]$ for $t \in [-\alpha, \alpha]$. This has a more complex distribution and four sharp edges. The random samples are visualized in Figure~\ref{fig:polygan_linear_astroid}. \modelname{} models the data distribution accurately in contrast to the two baselines. 

\begin{figure}[h]
    \subfloat[GT]{\includegraphics[width=0.28\linewidth]{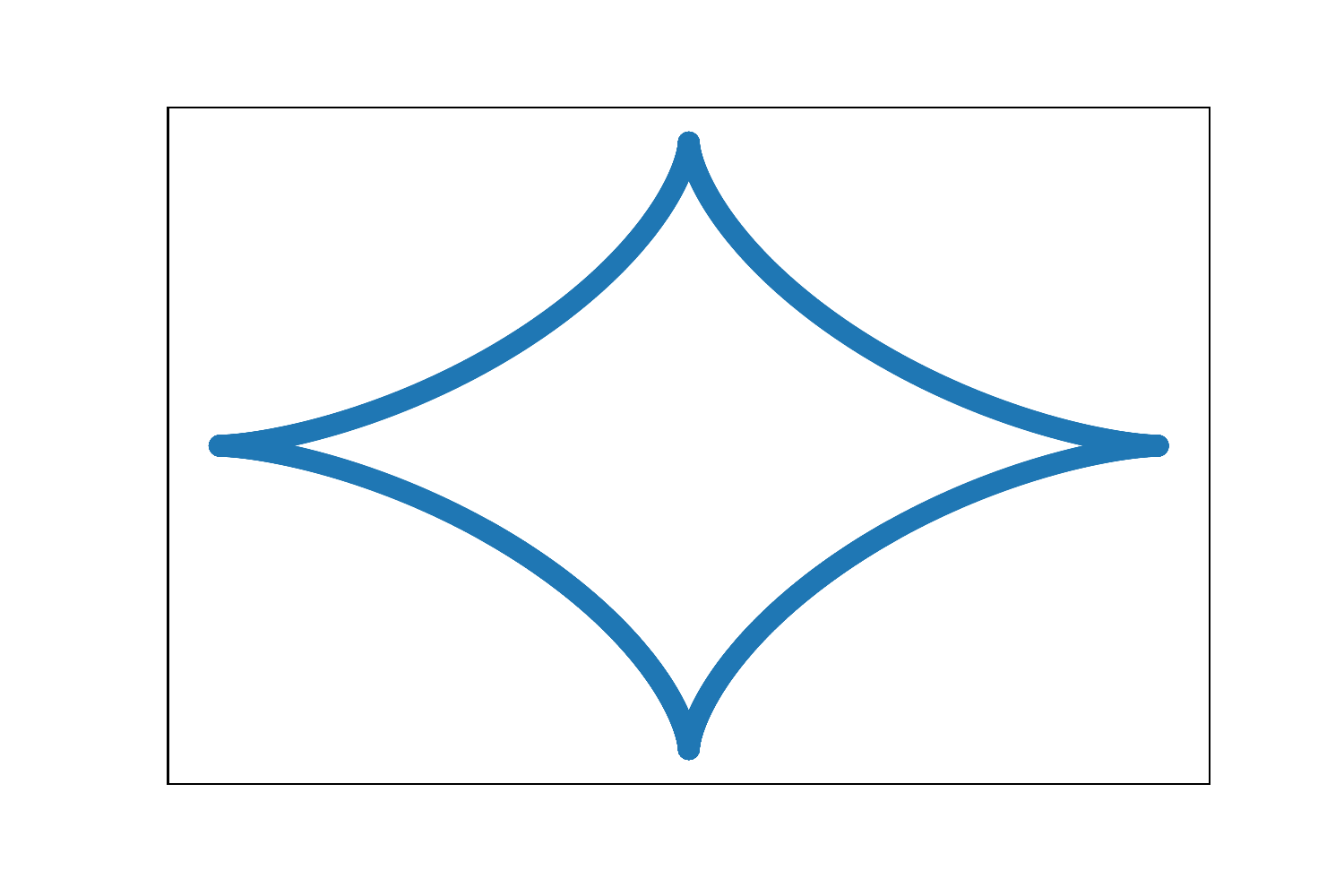}\hspace{-4mm}}
    \subfloat[Orig]{\includegraphics[width=0.28\linewidth]{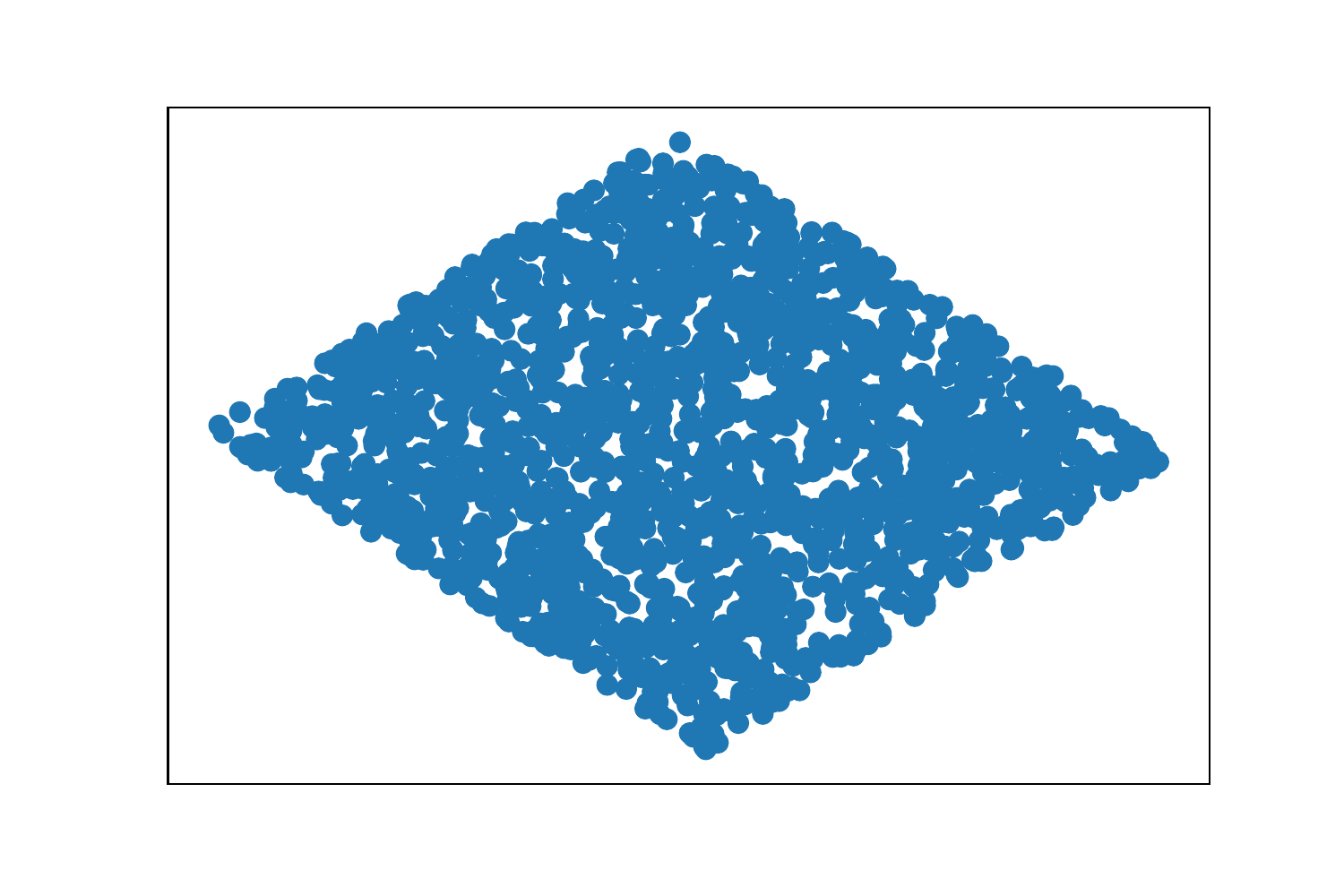}\hspace{-4mm}}
    \subfloat[Concat]{\includegraphics[width=0.28\linewidth]{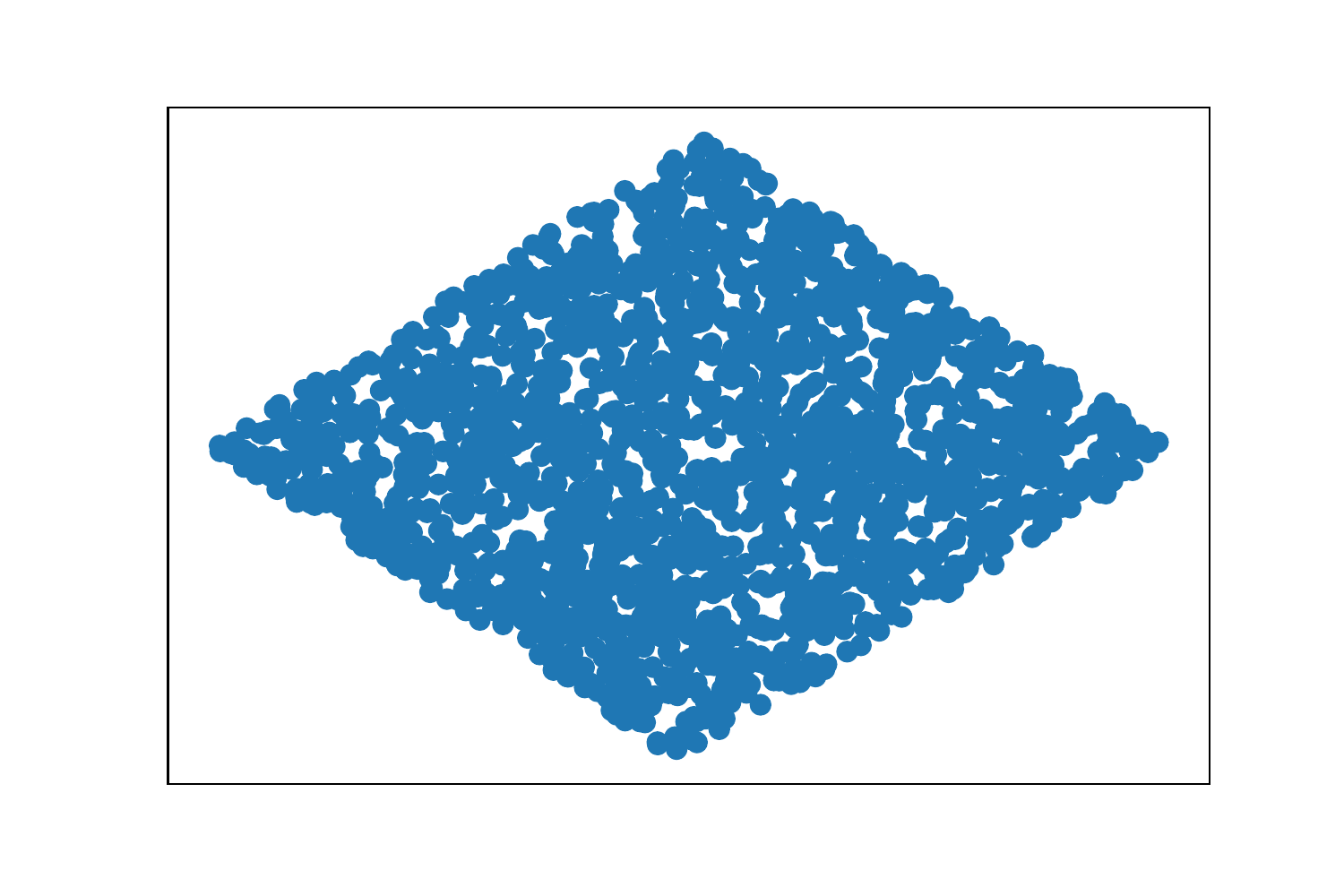}\hspace{-4mm}}
    \subfloat[\modelname]{\includegraphics[width=0.28\linewidth]{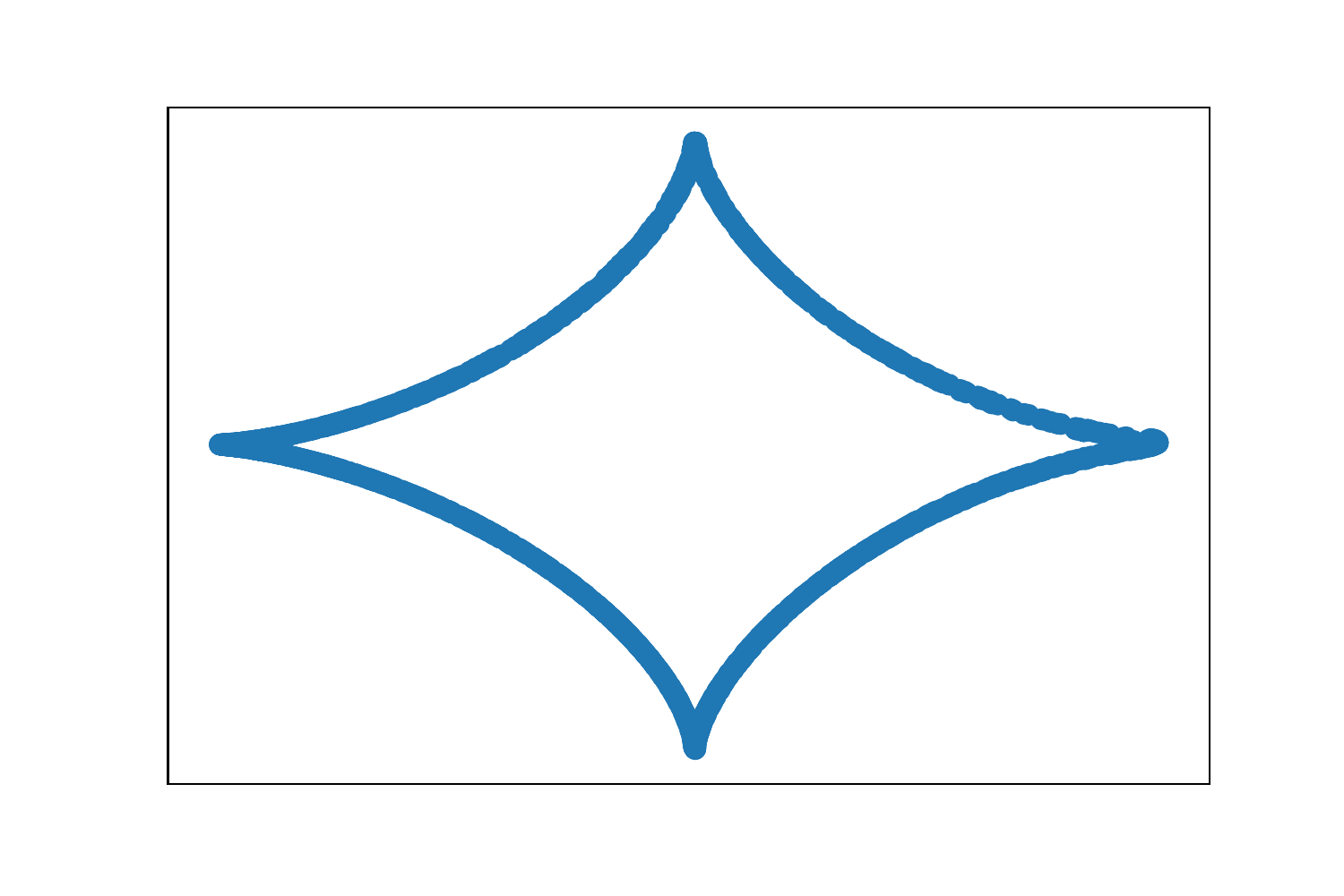}\hspace{-4mm}}
\caption{Synthesized data for learning the `astroid' signal. No activation functions are used in the generators.}
\label{fig:polygan_linear_astroid}
\end{figure}

\label{ssec:polygan_experiments_synthetic_suppl}

We conduct three experiments in which the data distribution is analytically derived. The experiments are: 

\textbf{Sin3D}: The data manifold is an extension over the 2D manifold of the sinusoidal experiment (Section~\ref{sec:polygan_synthetic_2d}). The function we want to learn is $G(\bm{z}): \realnum^2 \to \realnum^3$ with the data manifold described by the vector $[x, y, \sin(10 * \sqrt{x^2 + y^2})]$ for $x, y \in [-0.5, 0.5]$. 

In Figure~\ref{fig:polygan_linear_sin3d}, $20,000$ samples are sampled from the generators and visualized. \modelname{} captures the data distribution, while `Orig' and `Concat' fail.

\begin{figure}[h]
    \subfloat[GT]{\includegraphics[width=0.23\linewidth]{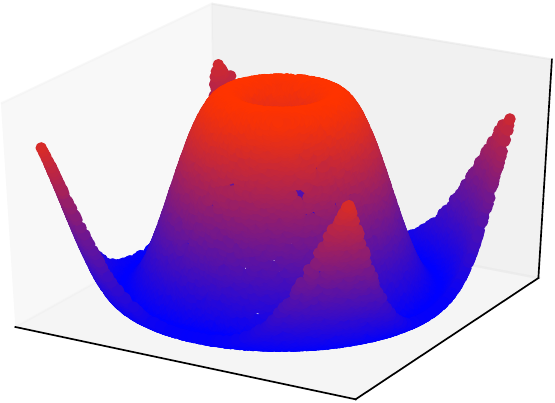}\hspace{+0.6mm}}
    \subfloat[Orig]{\includegraphics[width=0.23\linewidth]{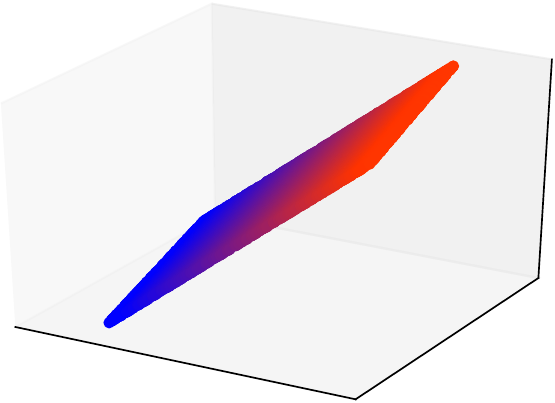}\hspace{+0.6mm}}
    \subfloat[Concat]{\includegraphics[width=0.23\linewidth]{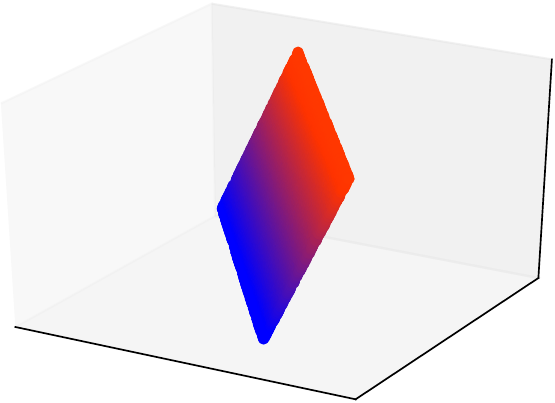}\hspace{+0.6mm}}
    \subfloat[\modelname]{\includegraphics[width=0.23\linewidth]{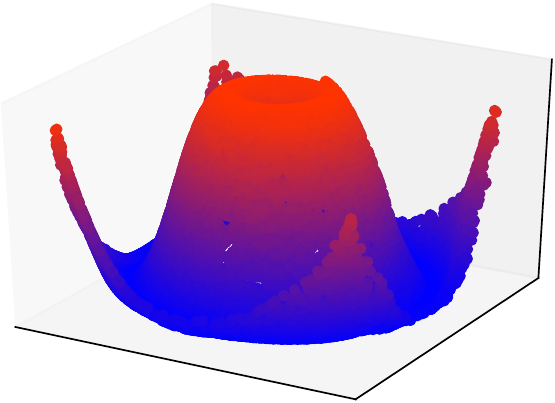}\hspace{+0.3mm}}
\caption{Experiment on 3D synthetic data. From left to right: (a) the data distribution, (b) `Orig', (c) `Concat', (d) \modelname. As expected, the `Orig' and the `Concat' cannot capture the data distribution.}
\label{fig:polygan_linear_sin3d}
\end{figure}

\textbf{Swiss roll}: The three dimensional vector $[t \cdot \sin{t}, y, t \cdot \cos{t}] + 0.05 \cdot s$ for $t, y \in [0, 1]$ and $s \sim \mathcal{N}(0, 1)$ forms the data manifold\footnote{This is a standard synthetic distribution in popular machine learning frameworks such as scikit-learn.}. In Figure~\ref{fig:polygan_linear_swiss_roll}, $20,000$ samples are visualized.

\begin{figure}[h]
    \subfloat[GT]{\includegraphics[width=0.23\linewidth]{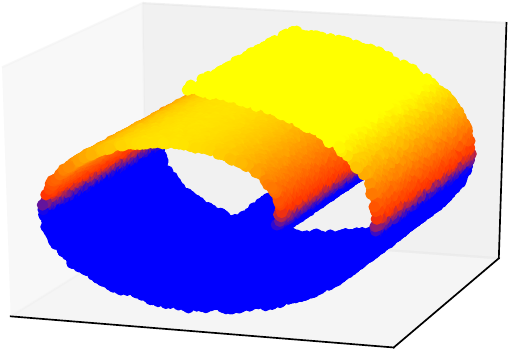}\hspace{+0.6mm}}
    \subfloat[Orig]{\includegraphics[width=0.23\linewidth]{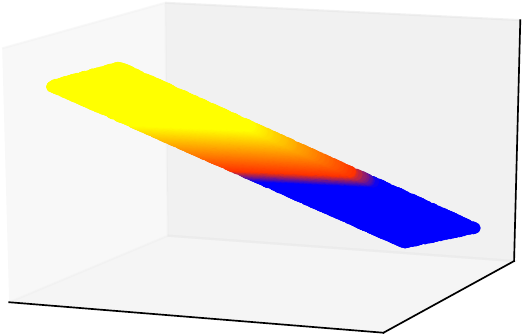}\hspace{+0.6mm}}
    \subfloat[Concat]{\includegraphics[width=0.23\linewidth]{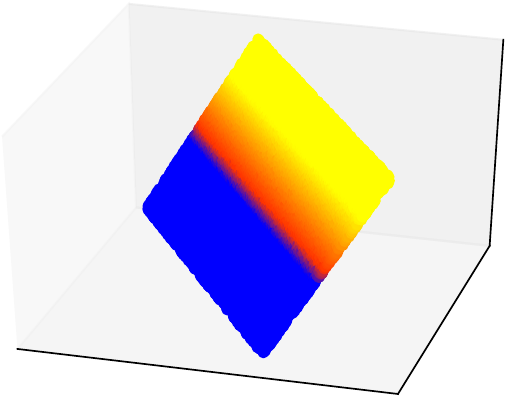}\hspace{+0.6mm}}
    \subfloat[\modelname]{\includegraphics[width=0.23\linewidth]{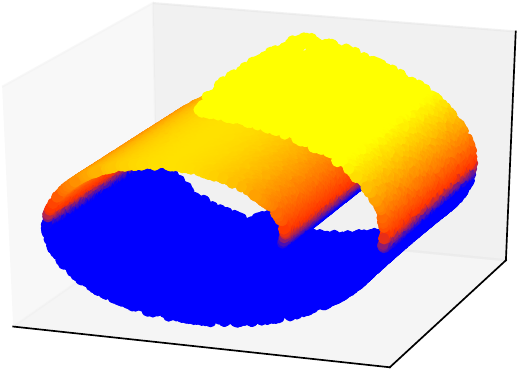}\hspace{+0.1mm}}
\caption{Experiment on 3D synthetic data (`swiss roll'). From left to right: (a) the data distribution, (b) `Orig', (c) `Concat', (d) \modelname.}
\label{fig:polygan_linear_swiss_roll}
\end{figure}

\textbf{Gabriel's Horn}: The three dimensional vector $[x, \alpha \cdot \frac{\cos{t}}{x}, \alpha \cdot \frac{\sin{t}}{x}]$ for $t \in [0, 160 \pi]$ and $x \in [1, 4]$ forms the data manifold. The dependence on both sinusoidal and the function $\frac{1}{x}$ makes this curve challenging for a polynomial expansion. In Figure~\ref{fig:polygan_linear_ghorn}, the synthesized samples are plotted. \modelname{} learns how to generate samples on the manifold despite the fraction in the parametric form.

\begin{figure}[h]
    \subfloat[GT]{\includegraphics[width=0.23\linewidth]{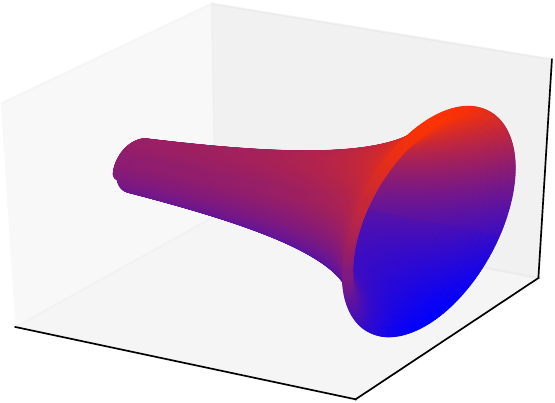}\hspace{+0.6mm}}
    \subfloat[Orig]{\includegraphics[width=0.23\linewidth]{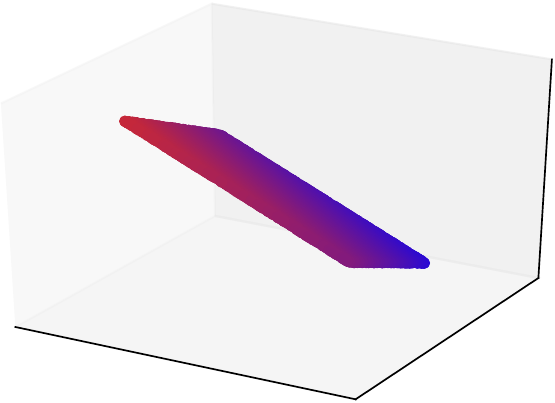}\hspace{+0.6mm}}
    \subfloat[Concat]{\includegraphics[width=0.23\linewidth]{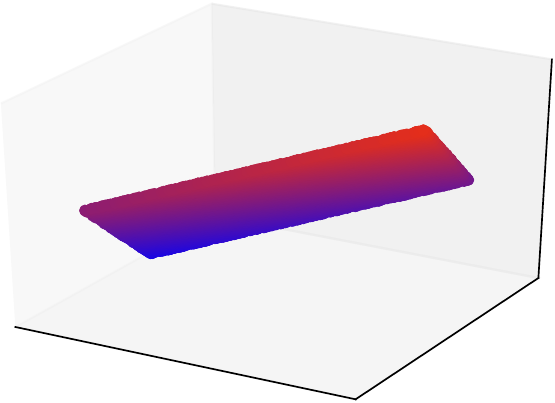}\hspace{+0.6mm}}
    \subfloat[\modelname]{\includegraphics[width=0.23\linewidth]{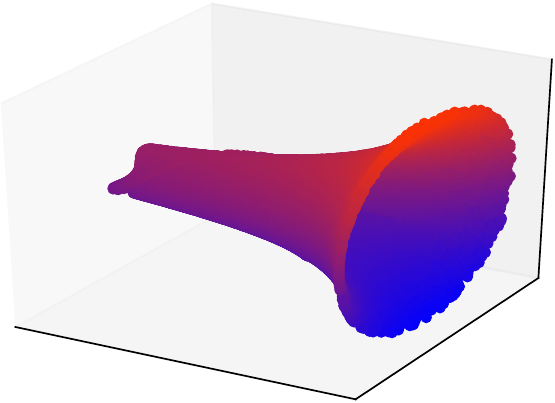}\hspace{+0.1mm}}
\caption{Synthesized data on `Gabriel's Horn'. From left to right: (a) the data distribution, (b) `Orig', (c) `Concat', (d) \modelname.}
\label{fig:polygan_linear_ghorn}
\end{figure}

 \section{Image generation with linear blocks}
\label{sec:polygan_image_generation_linear_suppl}

Apart from digit generation (Section~\ref{sec:polygan_linear_mnist}), we conduct two experiments on image generation of face and natural scenes. Since both distributions are harder than the digits one, we extend the approximation followed on Section~\ref{sec:polygan_linear_mnist} by one order, i.e., we assume a fifth-order approximation. We emphasize that each block is a residual block with no activation functions.

\textbf{Faces}: In the experiment with faces, we utilize as the training samples the YaleB~\citep{georghiades2001few} dataset. The dataset includes greyscale images of faces under extreme illuminations. We rescale all of the images into $64\times64$ for our analysis.

Random samples are illustrated in Figure~\ref{fig:polygan_linear_yaleb_samples}. Our method generates diverse images and captures the case of illuminating either half part of the face, while `Orig' and `Concat' generate images that have a dark side only on the left and right side, respectively. The difference becomes profound in the finer details of the face (please zoom in), where both baselines fail to synthesize realistic semantic parts of the face.

\begin{figure}[htb]
    \subfloat[GT]{\includegraphics[width=0.24\linewidth]{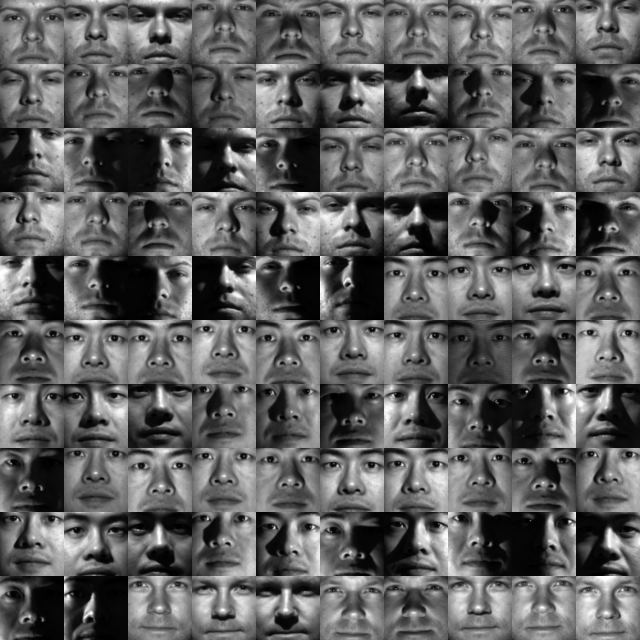}\hspace{1mm}}
    \subfloat[Orig]{\includegraphics[width=0.24\linewidth]{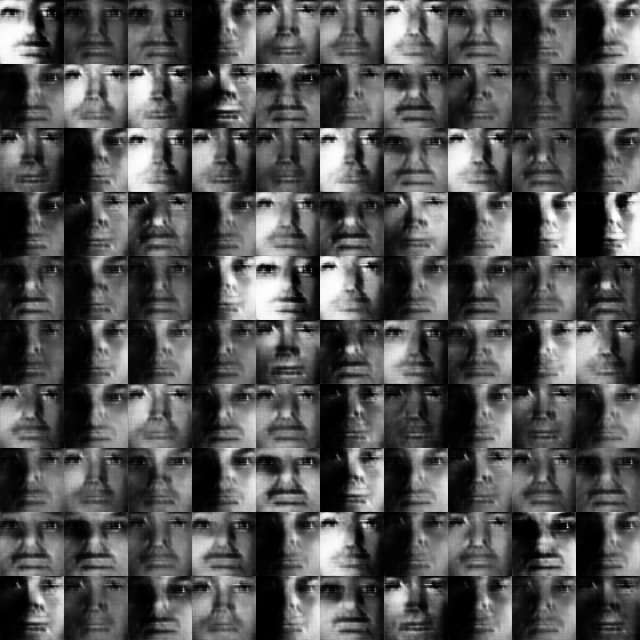}\hspace{1mm}}
    \subfloat[Concat]{\includegraphics[width=0.24\linewidth]{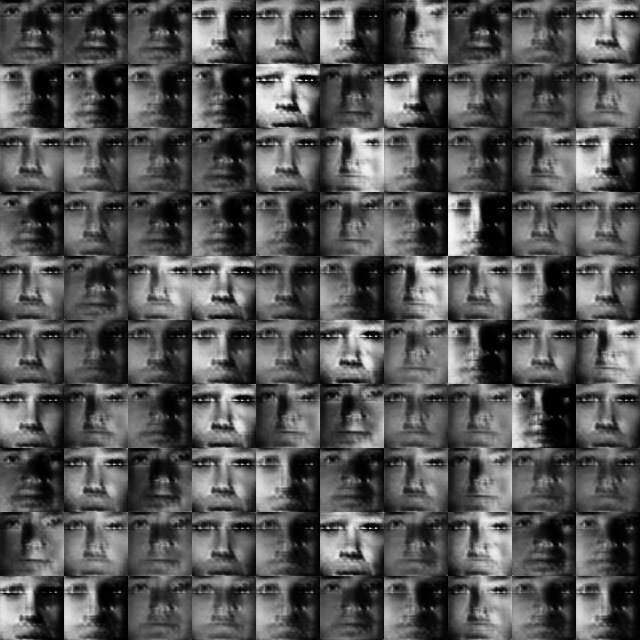}\hspace{1mm}}
    \subfloat[\modelname]{\includegraphics[width=0.24\linewidth]{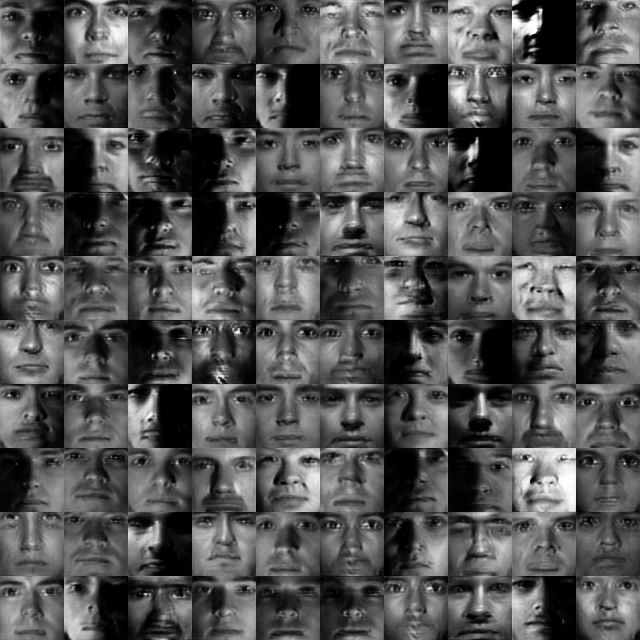}\hspace{1mm}}
\caption{Image generation on faces (YaleB~\citep{georghiades2001few}) for a generator with linear blocks and a single activation function only on the output (i.e., $\tan h$). Notice that our method can illuminate either the left or right part of the face, in contrast to  `Orig' (and `Concat') which  generate images that have a dark side only on the left (respectively right) side. In addition, both `Orig' and `Concat' fail to capture the fine details of the facial structure (please zoom in for the details).}
\label{fig:polygan_linear_yaleb_samples}
\end{figure}

\textbf{Natural scenes}: We further evaluate the generation of natural images, specifically by training on CIFAR10~\citep{krizhevsky2014cifar}. CIFAR10 includes $50,000$ training images of $32\times32\times3$ resolution.

In Table~\ref{tab:polygan_linear_conditional_cifar_quantitative}, we evaluate the standard metrics of Inception Score (IS) and Frechet Inception Distance (FID) (see more details for the metrics in section~\ref{sec:polygan_experiments_sota_networks}). Our model outperforms both `Orig' and `Concat' by a considerable margin. In Figure~\ref{fig:injection_gan_linear_cifar10_conditional}, some random synthesized samples are presented.

\begin{table}[h]
\centering
\caption{IS/FID scores on CIFAR10~\citep{krizhevsky2014cifar} with linear blocks.} 
 \begin{tabular}{|c | c | c|} 
     \hline
     \multicolumn{3}{|c|}{conditional SNGAN with linear blocks on CIFAR10}\\
     \hline
     Model & IS ($\uparrow$) & FID ($\downarrow$)\\
     \hline
     Orig & $4.47\pm 0.21$ & $156.67\pm 12.29$\\
     \hline
     Concat & $4.23\pm 0.37$ & $188.08\pm 17.00$\\
     \hline
     \modelname{} & $\bm{6.43\pm 0.11}$ & $\bm{53.50\pm 2.71}$\\
     \hline
 \end{tabular}
 \label{tab:polygan_linear_conditional_cifar_quantitative}
\end{table}

\begin{figure}[htb]
    \subfloat[GT]{\includegraphics[width=0.24\linewidth]{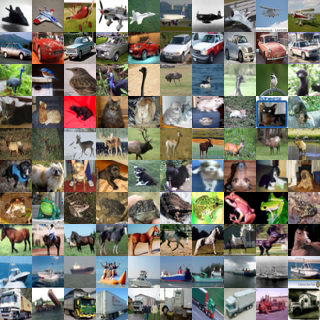}\hspace{1mm}}
    \subfloat[Orig]{\includegraphics[width=0.24\linewidth]{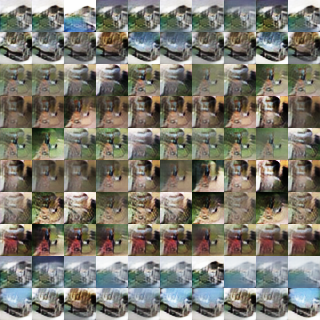}\hspace{1mm}}
    \subfloat[Concat]{\includegraphics[width=0.24\linewidth]{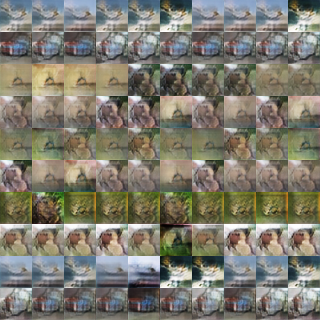}\hspace{1mm}}
    \subfloat[\modelname]{\includegraphics[width=0.24\linewidth]{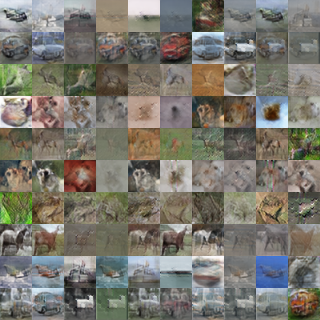}\hspace{1mm}}
\caption{Conditional image generation on CIFAR10 for a generator with linear blocks and a single activation function. Our approach generates more realistic samples in comparison to the compared methods, where severe mode collapse also takes place.}
\label{fig:injection_gan_linear_cifar10_conditional}
\end{figure}

 \section{Image generation with activation functions}
\label{sec:polygan_experiments_sota_networks}
To demonstrate the flexibility of the \modelname, we utilize three different popular generators. The three acrhitectures chosen are DCGAN~\citep{radford2015unsupervised}, SNGAN~\citep{miyato2018spectral}, and SAGAN~\citep{zhang2018self}. Each original generator is converted into a polynomial expansion, while we use the non-linearities to boost the performance of the polynomial generator.  The hyper-parameters are kept the same as the corresponding baseline. Algorithms \ref{alg:polygan_sngan_generator} and \ref{alg:polygan_sngan_generator_injection} succinctly present the key differences of our approach compared to the traditional one (in the case of SNGAN, similarly for other architectures). 

In addition to the baseline, we implement the most closely related alternative to our framework, namely instead of using the Hadamard operator as in Figure~\ref{fig:polygan_core_module}, we concatenate the noise with the feature representations at that block. The latter approach is frequently used in the literature~\citep{berthelot2017began, brock2019large} (referred as ``Concat'' in the paper). The number of the trainable parameters of the generators are reported in Table \ref{tab:polygan_table_params_generators}. Our method has only a minimal increase of the parameters, while the concatenation increases the number of parameters substantially.

To reduce the variance often observed during GAN training~\citep{lucic2018gans, odena2018generator}, each reported score is averaged over $10$ runs utilizing different seeds. The metrics we utilize are Inception Score (IS)~\citep{salimans2016improved} and Frechet Inception Distance (FID)~\citep{heusel2017gans}.

Below, we perform an ablation study on Section~\ref{ssec:polygan_experiments_ablation_suppl}, and then present the experiments on unsupervised (Section~\ref{ssec:polygan_experiments_unsupervised}) and conditional image generation (Section~\ref{ssec:polygan_experiments_conditional}) respectively. 

\textbf{Datasets}: We use CIFAR10~\citep{krizhevsky2014cifar} and Imagenet~\citep{russakovsky2015imagenet} as the two most widely used baselines for GANs:
\begin{itemize}
    \item \textbf{CIFAR10}~\citep{krizhevsky2014cifar} includes $60,000$ images of $32\times\ 32$ resolution. We use $50,000$ images for training and the rest for testing. 
    \item \textbf{Imagenet}~\citep{russakovsky2015imagenet} is a large scale dataset that includes over one million training images and $50,000$ validation images. We reshape the images to $128\times128$ resolution.
\end{itemize}

\textbf{Baseline architectures}: The architectures employed are: 
\begin{itemize}
    \item DCGAN~\citep{radford2015unsupervised}, as implemented in \url{https://github.com/pytorch/examples/tree/master/dcgan}. This is a widely used baseline.
    \item SNGAN~\citep{miyato2018spectral}, as implemented in \url{https://github.com/pfnet-research/sngan_projection}. SNGAN is a strong performing GAN that introduced a spectral normalization in the discriminator.
    \item SAGAN~\citep{zhang2018self}, as implemented in \url{https://github.com/voletiv/self-attention-GAN-pytorch}. This is a recent network architecture that utilizes the notion of self-attention \citep{wang2018non} in a GAN setting, achieving impressive results on Imagenet \citep{russakovsky2015imagenet}.
\end{itemize}

The default hyper-parameters are left unchanged. The aforementioned codes are used for reporting the results of both the baseline and our method to avoid any discrepancies, e.g. different frameworks resulting in unfair comparisons. The source code will be released to enable the reproduction of our results.

\textbf{Evaluation metrics}: The popular Inception Score (IS)~\citep{salimans2016improved} and Frechet Inception Distance (FID)~\citep{heusel2017gans} are used for the quantitative evaluation. Both scores extract feature representations from a pretrained classifier (in practice the Inception network~\citep{szegedy2015going}). 
Despite their shortcomings, IS and FID are widely used~\citep{lucic2018gans, creswell2018generative}, since alternative metrics fail for generative models~\citep{theis2015note}. 

The Inception Score is defined as 

\begin{equation}
\exp\left (\mathbb{E}_{\bm{x} \in P_\theta} [KL(p(y \vert \bm{x}) \Vert p(y)]  \right)
\end{equation}

where $\bm{x}$ is a generated sample and $p(y\vert \mathbf{x})$ is the conditional distribution for labels $y$. The distribution $p(y)$ over the labels is approximated by $\frac{1}{M}\sum_{n=1}^M p(y|\bm{x}_n)$ for $\bm{x}_n$ generated samples. Following the methods in the literature~\citep{miyato2018spectral}, we compute the inception score for $M=5,000$ generated samples per run ($10$ splits for each run). 

The Frechet Inception Distance (FID) utilizes feature representations from a pretrained network~\citep{szegedy2015going} and assumes that the distributions of these representations are Gaussian. Denoting the representations of real images as $\mathcal{N}(\bm{\mu}_r, \bm{C}_r)$ and the generated (fake) as $\mathcal{N}(\bm{\mu}_f, \bm{C}_f)$, FID is:

\begin{equation}
\|\bm{\mu}_r-\bm{\mu}_f\|_2^2+ trace \bigl(\bm{C}_r+\bm{C}_f-2\bigl(\bm{C}_r \bm{C}_f\bigr)^{1/2}\big)
\end{equation}

In the experiments, we use $M=10,000$ to compute the mean and covariance of the real images and $M=10,000$ synthesized samples for $\bm{\mu}_f, \bm{C}_r$.

For both scores the original tensorflow inception network weights are used; the routines of tensorflow.contrib.gan.eval are called for the metric evaluation.

\subsection{Implementation details}
\label{ssec:polygan_experiments_implementation}

We experimentally define that a (series of) affine transformation(s) on the input noise $\bm{z}$ are beneficial before using the transformed $\bm{z}$ for the Hadamard products.\footnote{A similar transformation is performed in other GAN architectures, such as in \citet{karras2018style}.} These affine transformations are henceforth mentioned as \emph{global transformations} on $\bm{z}$.

The implementation details for each network are the following: 
\begin{itemize}
    \item DCGAN: We use a global transformation followed by a RELU non-linearity. WThe rest details remain the same as the baseline model.
    \item SNGAN: Similarly to DCGAN, we use a global transformation with a RELU non-linearity. 
\end{itemize}

\begin{table}
\begin{tabular}{l  l}
\parbox{0.48\textwidth}{
\begin{algorithm}[H]
	\begin{small}
	\SetKwInOut{Input}{Input}
	\SetKwInOut{Output}{Output}
	\SetKwInOut{Init}{Initialize}
	\Input{Noise $\bm{z} \in \realnum^{128}$, $\phi = $ RELU}
	\Output{$\bm{x} \in \realnum^{32 \times 32 \times 3}$}
	    \;
	    \;
        \% fully-connected layer for reshaping.
        
        $\bm{h}$ = $\phi$(Linear($\bm{z}$))  \% dims out: $4 \times 4 \times 256 $.\; \;\;\;
        
        \For{n=1:3}{
        
            \% resnet blocks.
            
            $\bm{h}$ = resblock($\bm{h}$)  
            
            \% dims out: $ (4 \cdot 2 ^ n) \times (4 \cdot 2 ^ n) \times 256 $.\;\;\; \; 
        }
            
        $\bm{x}$ = tanh(Conv($\bm{h}$)) \% dims out: $ 32 \times 32 \times 3 $.
        \caption{Original SNGAN generator.}
	\label{alg:polygan_sngan_generator}
	\end{small}
\end{algorithm}

}
&
\parbox{0.48\textwidth}{
\begin{algorithm}[H]
	\begin{small}
	\SetKwInOut{Input}{Input}
	\SetKwInOut{Output}{Output}
	\SetKwInOut{Init}{Initialize}
	\Input{Noise $\bm{z} \in \realnum^{128}$, $\phi = $ RELU}
	\Output{$\bm{x} \in \realnum^{32 \times 32 \times 3}$}

        \% global transformation of $\bm{z}$.
        
        \algcol{$\bm{v}$ = $\phi$(Linear($\bm{z}$))}
        
        \% fully-connected layer for reshaping.
        
        $\bm{h}$ = $\phi$(Linear($\bm{v}$))  \% dims out: $4 \times 4 \times 256 $.
        
        \% perform a hadamard product here.
        
        \algcol{$\bm{v}_0 = (\bm{A}\matnot{0})^T \cdot \bm{v}$}
        
        \algcol{$\bm{h}$ = $\bm{h} * \bm{v}_0$} \% dims out: $4 \times 4 \times 256 $.
        
        \For{n=1:3}{
        
            \% resnet blocks.
            
            $\bm{h}$ = resblock($\bm{h}$)  
            
            \% dims out: $ (4 \cdot 2 ^ i) \times (4 \cdot 2 ^ i) \times 256 $.
            
            \% reshape $\bm{v}$ for hadamard product.
            
            \algcol{$\bm{v}_n = (\bm{A}\matnot{n})^T \cdot \bm{v}$}
            
            \algcol{$\bm{h}$ = $\bm{h} * \bm{v}_n$}
        }
        $\bm{x}$ = tanh(Conv($\bm{h}$)) \% dims out: $ 32 \times 32 \times 3 $.
        \caption{Modified SNGAN-poly.}
	\label{alg:polygan_sngan_generator_injection}
	\end{small}
\end{algorithm}
}
\end{tabular}
\caption{The algorithm on the left describes the SNGAN generator. The algorithm on the right preserves the resnet blocks of the SNGAN generator, but converts it into a polynomial (named SNGAN-poly). The different lines are emphasized with blue color.}
\end{table}

We consider each residual block as one order of approximation and compute the Hadamard product after each block (see algorithm~\ref{alg:polygan_sngan_generator_injection}).

\subsection{Ablation study on CIFAR10 with non-linear generators}
\label{ssec:polygan_experiments_ablation_suppl}

We conduct an ablation study based on SNGAN architecture (or our variant of SNGAN-poly), since most recent methods are based on similar generators~\cite{zhang2018self, brock2019large}. Unless explicitly mentioned otherwise, the SNGAN is trained on CIFAR10 for unsupervised image generation.

\textbf{Global transformation validation}: We add a global transformation on $\bm{z}$, i.e. a fully-connected layer and use the transformed noise as input to the generator. In the first experiment, we evaluate whether to add a non-linear activation to the global transformation. The two alternatives are: i) with linear global transformation (`Ours-linear-global'), i.e. no non-linearity, and ii) with global transformation followed by a RELU non-linearity (`Ours-RELU-global').

\begin{table}[h]
\centering
\caption{Global transformation validation on SNGAN. The first two results assess the addition of a non-linear activation function after the global transformation. The last two rows compare the addition of a global transformation on the original generator.} 
 \begin{tabular}{|c | c | c|} 
     \hline
     \multicolumn{3}{|c|}{SNGAN on CIFAR10}\\ 
     \hline
     Model & IS ($\uparrow$) & FID ($\downarrow$)\\
     \hline
     Ours-linear-global & $8.23\pm0.10$ & $18.85\pm 0.59$\\
     \hline
     Ours-RELU-global & $\bm{8.30\pm 0.09}$ & $\bm{17.65\pm 0.76}$\\
     \specialrule{.14em}{.07em}{.07em}
     Original & $8.06\pm 0.10$ & $19.06\pm 0.50$\\
     \hline
     Original-RELU-global & $7.98\pm0.27$ & $37.61\pm 7.16$\\
     \hline
 \end{tabular}
 \label{tab:polygan_ablation_sngan_global_trans}
\end{table}

The first two results in Table~\ref{tab:polygan_ablation_sngan_global_trans} demonstrate that both metrics marginally improve when using a non-linear activation function. We add this global transformation with RELU on the original SNGAN. The results are reported in the last two rows of Table~\ref{tab:polygan_ablation_sngan_global_trans} (where the original is mentioned as `Orig', while the alternative of adding a global transformation as `Original-RELU-global').

\textbf{Split $\bm{z}$ into chunks}: The recent BigGAN of \citep{brock2019large} performs hierarchical synthesis of images by splitting the latent vector $\bm{z}$ into one chunk per resolution (block). Each chunk is then concatenated into the respective resolution. 

We scrutinize this splitting against our method; we split the noise $\bm{z}$ into $(k+1)$ non-overlapping chunks of equal size for performing $k$ injections. The injection with splitting is mentioned as `Inject-split' below. Our splitting deteriorates the scores on the task as reported in Table~\ref{tab:polygan_ablation_spliting_z}. It is possible that more elaborate splitting techniques, such as those in \citet{brock2019large} are beneficial.

\begin{table}[h]
\centering
\caption{Ablation experiment on splitting the noise $\bm{z}$ into non-overlapping chunks for the injection.} 
 \begin{tabular}{|c | c | c|} 
     \hline
     \multicolumn{3}{|c|}{SNGAN on CIFAR10}\\
     \hline
     Model & IS ($\uparrow$) & FID ($\downarrow$)\\
     \hline
     Original & $8.06\pm 0.10$ & $19.06\pm 0.50$\\
     \hline
     Inject-split & $7.75\pm0.12$ & $22.08\pm 0.98$\\
     \hline
     Ours-RELU-global & $\bm{8.30\pm 0.09}$ & $\bm{17.65\pm 0.76}$\\
     \hline
 \end{tabular}
 \label{tab:polygan_ablation_spliting_z}
\end{table}

\textbf{Normalization before Hadamard product}: In \citet{karras2018style} they normalize the transformed noise through ADAIN, while in \citet{karras2017progressive} they similarly perform a feature vector normalization. 

We scrutinize a feature normalization on the baseline of `Ours-RELU-global'. For each layer $i$ we divide the $\bm{A}\matnot{i} \bm{z}$ vector with its standard deviation. The variant with global transformation followed by RELU and normalization before the Hadamard product is called `Ours-norm'. The results in Table~\ref{tab:polygan_ablation_normalization_before_hadamard} illustrate that normalization improves the metrics.

\begin{table}[h]
\centering
\caption{Ablation experiment on normalizing the $\bm{A}\matnot{i} \bm{z}$ vector before the Hadamard product.} 
 \begin{tabular}{|c | c | c|} 
     \hline
     \multicolumn{3}{|c|}{SNGAN on CIFAR10}\\
     \hline
     Model & IS ($\uparrow$) & FID ($\downarrow$)\\
     \hline
     Ours-RELU-global & $8.30\pm 0.09$ & $17.65\pm 0.76$\\
     \hline
     Ours-norm & $\bm{8.37\pm 0.11}$ & $\bm{17.14\pm 0.58}$\\
     \hline
 \end{tabular}
 \label{tab:polygan_ablation_normalization_before_hadamard}
\end{table}

\textbf{Skip the Hadamard product}: Motivated by the skip connection of our \modelone{}, we add a skip connection to each Hadamard product. For instance, we modify the term $\Big( (\bm{A}\matnot{1})^T \bm{z} \Big) * \Big( (\bm{B}\matnot{1})^T \bm{b}\matnot{1} \Big)$ into $\Big( (\bm{A}\matnot{1})^T \bm{z} \Big) * \Big( (\bm{B}\matnot{1})^T \bm{b}\matnot{1} \Big) + (\bm{B}\matnot{1})^T \bm{b}\matnot{1}$.

In Table~\ref{tab:polygan_ablation_skip_hadamard}, we use `Ours-RELU-global' as baseline against the model with the skip connection (`Ours-skip').

\begin{table}[h]
\centering
\caption{Ablation experiment on adding a skip connection to each Hadamard product.} 
 \begin{tabular}{|c | c | c|} 
     \hline
     \multicolumn{3}{|c|}{SNGAN on CIFAR10}\\
     \hline
     Model & IS ($\uparrow$) & FID ($\downarrow$)\\
     \hline
     Ours-RELU-global & $8.30\pm 0.09$ & $\bm{17.65\pm 0.76}$\\
     \hline
     Ours-skip & $\bm{8.43\pm 0.11}$ & ${21.54\pm 1.59}$\\
     \hline
 \end{tabular}
 \label{tab:polygan_ablation_skip_hadamard}
\end{table}

Since we use SNGAN both for unsupervised/conditional image generation, we verify the aforementioned results in the \emph{conditional setting}, i.e. when the class information is also provided to the generator and the discriminator.

\textbf{Normalization before Hadamard product}: Similarly to the experiment above, for each layer $i$ we divide the $\bm{A}\matnot{i} \bm{z}$ vector with its standard deviation. The quantitative results in Table~\ref{tab:polygan_ablation_normalization_before_hadamard_conditional} improve the IS score, but the FID deteriorates.

\begin{table}[h]
\centering
\caption{Ablation experiment (conditional GAN setting) on normalizing the $\bm{A}\matnot{i} \bm{z}$ vector before the Hadamard product.} 
 \begin{tabular}{|c | c | c|} 
     \hline
     \multicolumn{3}{|c|}{conditional SNGAN on CIFAR10}\\
     \hline
     Model & IS ($\uparrow$) & FID ($\downarrow$)\\
     \hline
     Ours-RELU-global & ${8.66\pm 0.14}$ & $\bm{13.52\pm0.60}$\\
     \hline
     Ours-norm & $\bm{8.76\pm 0.11}$ & ${15.40\pm 1.29}$\\
     \hline
 \end{tabular}
 \label{tab:polygan_ablation_normalization_before_hadamard_conditional}
\end{table}

\textbf{Skip the Hadamard product}: Similarly to the aforementioned unsupervised case, we assess the performance if we add a skip connection in the Hadamard. In Table~\ref{tab:polygan_ablation_skip_hadamard_conditional}, the quantitative results comparing the baseline and the skip case are presented.

\begin{table}[h]
\centering
\caption{Ablation experiment (conditional GAN setting) on adding a skip connection to each Hadamard product.} 
 \begin{tabular}{|c | c | c|} 
     \hline
     \multicolumn{3}{|c|}{conditional SNGAN on CIFAR10}\\
     \hline
     Model & IS ($\uparrow$) & FID ($\downarrow$)\\
     \hline
     Ours-RELU-global & ${8.66\pm 0.14}$ & $\bm{13.52\pm0.60}$\\
     \hline
     Ours-skip & $\bm{8.77\pm 0.10}$ & ${13.62\pm 0.69}$\\
     \hline
 \end{tabular}
 \label{tab:polygan_ablation_skip_hadamard_conditional}
\end{table}

\subsection{Unsupervised image generation}
\label{ssec:polygan_experiments_unsupervised}
In this experiment, we study the image generation problem without any labels or class information for the images. The architectures of DCGAN and resnet-based SNGAN are used for image generation in CIFAR10~\citep{krizhevsky2014cifar}. Table \ref{tab:polygan_table_unsupervised_quantitative} summarizes the results of the IS/FID scores of the compared methods. In all of the experiments, PolyGAN outperforms the compared methods.

\begin{table}[t]
\centering
\renewcommand{\arraystretch}{1.3}
 \caption{
IS/FID scores on CIFAR10~\citep{krizhevsky2014cifar} utilizing DCGAN~\citep{radford2015unsupervised} and SNGAN~\citep{miyato2018spectral} architectures for unsupervised image generation. Each network is run for $10$ times and the mean and standard deviation are reported. In both cases, inserting block-wise noise injections to the generator (i.e., converting to our proposed PolyGAN) results in an improved score. Higher IS / lower FID score indicate better performance. \\}
 \begin{minipage}{.5\linewidth}
 \begin{tabular}{|c | c | c|} 
 \hline
 \multicolumn{3}{|c|}{DCGAN}\\ 
 \hline
 Model & IS ($\uparrow$) & FID ($\downarrow$)\\
 \hline
 Orig & $6.25\pm 0.06$ & $47.29\pm 2.06$\\
 \hline
 Concat & $6.03\pm 0.06$ & $49.35\pm 2.17$\\
 \hline
 \modelname & $\bm{6.61\pm 0.05}$ & $\bm{42.86\pm 1.02}$\\
 \hline
 \end{tabular}
 \end{minipage}
  \begin{minipage}{.5\linewidth}
 \begin{tabular}{|c | c | c|} 
 \hline
 \multicolumn{3}{|c|}{SNGAN}\\ 
 \hline
 Model & IS ($\uparrow$) & FID ($\downarrow$)\\
 \hline
 Orig & $8.06\pm 0.10$ & $19.06\pm 0.50$\\
 \hline
 Concat & $8.28\pm 0.16$ & $20.77\pm 2.91$\\
 \hline
 \modelname & $\bm{8.30\pm 0.09}$ & $\bm{17.65\pm 0.76}$\\
  \hline
 \end{tabular}
 \end{minipage}
\label{tab:polygan_table_unsupervised_quantitative}
 \end{table}

\subsection{Conditional image generation}
\label{ssec:polygan_experiments_conditional}
Frequently class information is available. We can utilize the labels, e.g. use conditional batch normalization or class embeddings, to synthesize images conditioned on a class. We train two networks, i.e., SNGAN~\citep{miyato2018spectral} in CIFAR10~\citep{krizhevsky2014cifar} and SAGAN~\citep{zhang2018self} in Imagenet~\citep{russakovsky2015imagenet}. SAGAN uses self-attention blocks~\citep{wang2018non} to improve the resnet-based generator.  

Despite our best efforts to show that our method is both architecture and database agnostic, the recent methods are run for hundreds of thousands or even million iterations till ``convergence''. In SAGAN the authors report that for each training multiple GPUs need to be utilized for weeks to reach the final reported Inception Score.  We report the metrics for networks that are run with batch size $64$ (i.e., four times less than the original $256$) to fit in a single 16GB NVIDIA V100 GPU. Following the current practice in ML, due to the lack of computational budget \citep{hoogeboom2019emerging}, we run SAGAN for $400,000$ iterations (see Figure~3 of the original paper for the IS during training)\footnote{Given the batch size difference, our training corresponds to roughly the $100,000$ steps of the authors' reported results.}. Each such experiment takes roughly $6$ days to train. The FID/IS scores of our approach compared against the baseline method can be found in Table~\ref{tab:polygan_table_conditional_quantitative}. In both cases, our proposed method yields a higher Inception Score and a lower FID. 

\begin{table}[t]
\centering
\renewcommand{\arraystretch}{1.3}
 \caption{ Quantitative results on conditional image generation. We implement both SNGAN trained on CIFAR10 and SAGAN trained on Imagenet (for $400,000$ iterations). Each network is run for 10 times and the mean and variance are reported. \\}
 \begin{minipage}{.5\linewidth}
 \begin{tabular}{|c | c | c|} 
 \hline
 \multicolumn{3}{|c|}{SNGAN (CIFAR10)}\\
 \hline
 Model & IS ($\uparrow$) & FID ($\downarrow$)\\
 \hline
 Orig & $8.30\pm0.11$ & $14.70\pm 0.97$\\
 \hline
 \modelname & $\bm{8.66\pm 0.14}$ & $\bm{13.52\pm0.60}$\\
 \hline
 \end{tabular}
 \end{minipage}
  \begin{minipage}{.5\linewidth}
 \begin{tabular}{|c | c | c|} 
 \hline
 \multicolumn{3}{|c|}{SAGAN (Imagenet)}\\ 
 \hline
 Model & IS ($\uparrow$) & FID ($\downarrow$)\\
 \hline
 Orig & $13.81 \pm 0.21$ & $138.20 \pm 8.71$\\
 \hline
 \modelname & $\bm{14.60 \pm 0.15}$ & $\bm{84.37 \pm 6.37}$\\
 \hline
 \end{tabular}
 \end{minipage}
\label{tab:polygan_table_conditional_quantitative}
 \end{table}
 
 \begin{table}[t]
 \renewcommand{\arraystretch}{1.3}
 \caption{Number of parameters for the generators of each approach and on various databases. As can be seen, our method only marginally increases the parameters while substantially improving the performance. On the other hand, ``Concat'' significantly increases the parameters without analogous increase in the performance.\\}
\centering
  \begin{minipage}{.3\linewidth}
 \begin{tabular}{|c | c|} 
 \hline
 \multicolumn{2}{|c|}{DCGAN (CIFAR10)}\\ 
 \hline
 Model & Params\\
 \hline
 Orig & $3,573,440$\\
 \hline
 Concat & $6,416,448$\\
 \hline
 \modelname & $3,663,936$\\
 \hline
 \end{tabular}
 \end{minipage}
 \begin{minipage}{.3\linewidth}
 \begin{tabular}{|c | c|} 
 \hline
 \multicolumn{2}{|c|}{SNGAN (CIFAR10)}\\ 
 \hline
 Model & Params\\
 \hline
 Orig & $4,276,739$\\
 \hline
 Concat & $6,383,875$\\
 \hline
 \modelname & $4,408,835$\\
 \hline
 \end{tabular}
 \end{minipage}
  \begin{minipage}{.3\linewidth}
 \begin{tabular}{|c | c|} 
 \hline
 \multicolumn{2}{|c|}{SAGAN (Imagenet)}\\ 
 \hline
 Model & Params\\
 \hline
 Orig & $42,079,300$\\
 \hline
 \modelname & $42,351,748$\\
 \hline
 \end{tabular}
 \end{minipage}
 \label{tab:polygan_table_params_generators}
\end{table}
 \section{Experimental model comparison}
\label{sec:polygan_suppl_model_comparison}

An experimental comparison of the two models described in Section~\ref{sec:polygan_method} is conducted below. Unless explicitly mentioned otherwise, the networks used below do not include any non-linear activation functions, they are polynomial expansions with linear blocks. We use the following four experiments: 

\textbf{Sinusoidal on 2D}: The data distribution is described by $[x, \sin(x)]$ with $x \in [0, 2\pi]$ (see Section~\ref{sec:polygan_synthetic_2d} for further details). We assume $8^{th}$ order approximation for \modelone{} and $12^{th}$ order for \modeltwo{}. Both have width $15$ units. The comparison between the two models in Figure~\ref{fig:polygan_linear_sin_model_comparison} demonstrates that they can both capture the data manifold. Impressively, the \modelone{} does not synthesize a single point that is outside of the manifold.

\begin{figure}[h]
    \subfloat[GT]{\includegraphics[width=0.36\linewidth]{linear-sin_gt.pdf}\hspace{-5mm}}
    \subfloat[\modelone]{\includegraphics[width=0.36\linewidth]{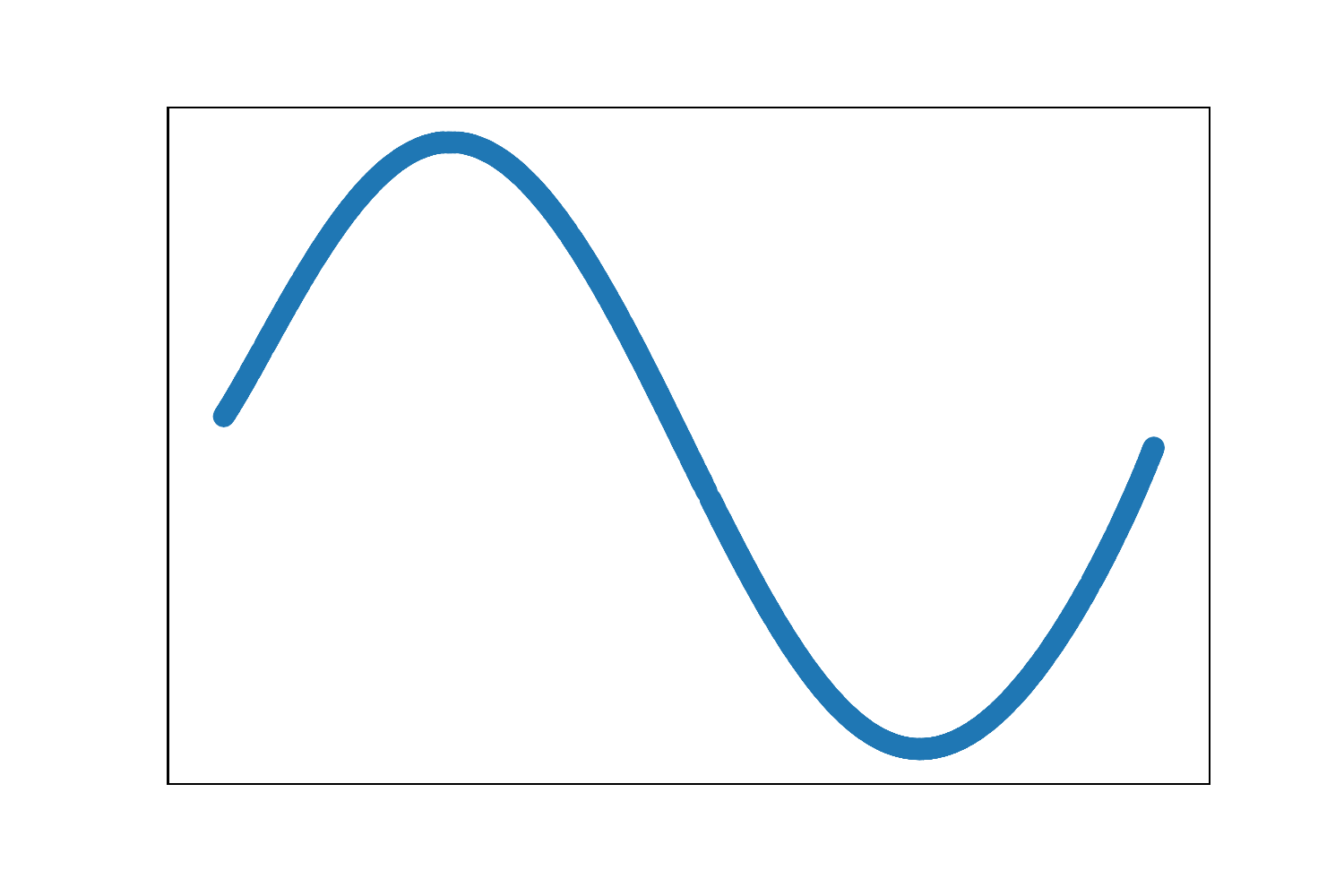}\hspace{-5mm}}
    \subfloat[\modeltwo]{\includegraphics[width=0.36\linewidth]{linear-sin_inject.pdf}\hspace{-4mm}}
\caption{Synthesized data for learning the $[x, sin(x)]$ signal. No activation functions are used in the generators. From left to right: (a) the data distribution, (b) \modelone{}, (c) \modeltwo{}.}
\label{fig:polygan_linear_sin_model_comparison}
\end{figure}

\textbf{Astroid}: The data distribution is described on Section~\ref{sec:polygan_experiments_suppl}. The samples comparing the two models are visualized in Figure~\ref{fig:polygan_linear_sin3d_model_comparison}.

\begin{figure}[h]
    \subfloat[GT]{\includegraphics[width=0.36\linewidth]{linear-astroid.pdf}\hspace{-5mm}}
    \subfloat[\modelone]{\includegraphics[width=0.36\linewidth]{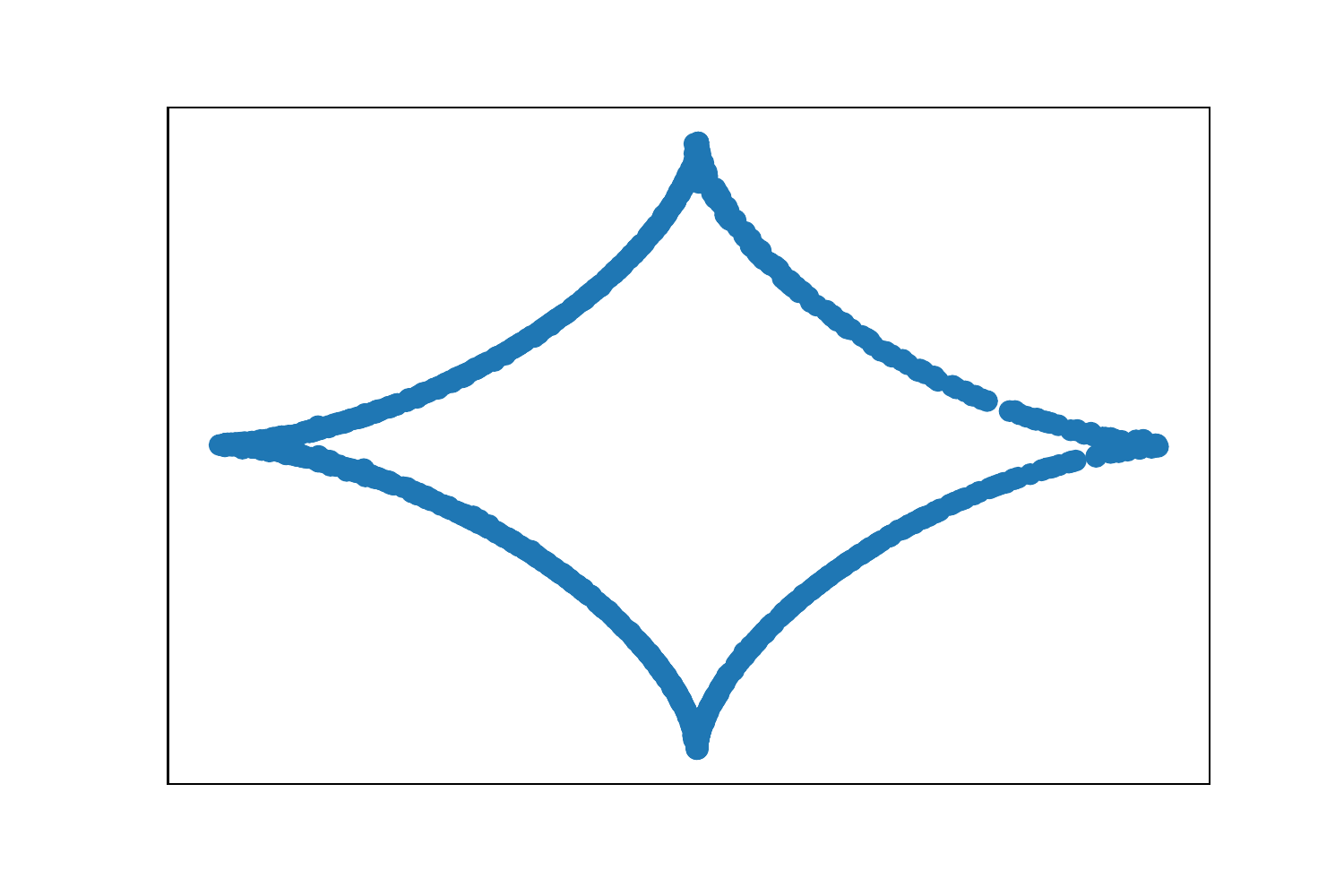}\hspace{-5mm}}
    \subfloat[\modeltwo]{\includegraphics[width=0.36\linewidth]{linear-astroid_inject.pdf}\hspace{-4mm}}
\caption{Synthesized data for learning the Astroid. Both models generate only plausible examples.}
\label{fig:polygan_linear_astroid_model_comparison}
\end{figure}

\textbf{Sin3D}: The data distribution is described on Section~\ref{ssec:polygan_experiments_synthetic_suppl}. In Figure~\ref{fig:polygan_linear_sin3d_model_comparison} the samples from the two models are illustrated.

\begin{figure}[h]
    \subfloat[GT]{\includegraphics[width=0.31\linewidth]{linear_synthetic_3d-sin3d_gt.png}\hspace{+0.6mm}}
    \subfloat[\modelone]{\includegraphics[width=0.31\linewidth]{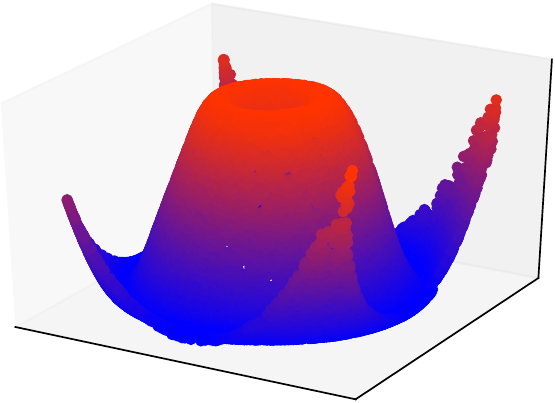}\hspace{+0.6mm}}
    \subfloat[\modeltwo]{\includegraphics[width=0.31\linewidth]{linear_synthetic_3d-sin3d_ours.png}}
\caption{Experiment on 3D synthetic data. From left to right: (a) the data distribution, (b) \modelone, (c) \modeltwo.}
\label{fig:polygan_linear_sin3d_model_comparison}
\end{figure}

\textbf{Swiss roll}: The data distribution is described on Section~\ref{ssec:polygan_experiments_synthetic_suppl}. In Figure~\ref{fig:polygan_linear_swiss_roll_model_comparison} the samples from the two models are illustrated.

\begin{figure}[h]
    \subfloat[GT]{\includegraphics[width=0.31\linewidth]{linear_synthetic_3d-swiss_roll_gt.png}\hspace{+0.6mm}}
    \subfloat[\modelone]{\includegraphics[width=0.31\linewidth]{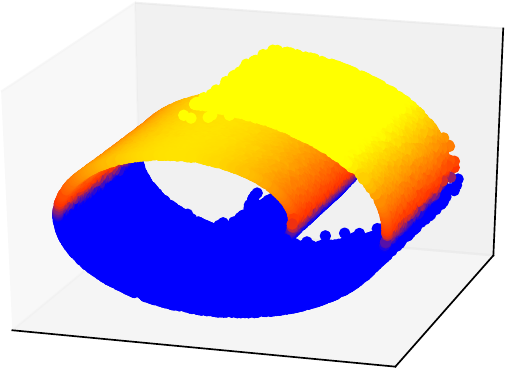}\hspace{+0.6mm}}
    \subfloat[\modeltwo]{\includegraphics[width=0.31\linewidth]{linear_synthetic_3d-swiss_roll_ours.png}\hspace{+0.1mm}}
\caption{Experiment on 3D synthetic data (`swiss roll'). From left to right: (a) the data distribution, (b) \modelone, (c) \modeltwo. Note that \modelone generates some noisy samples in contrast to \modeltwo.}
\label{fig:polygan_linear_swiss_roll_model_comparison}
\end{figure}

\textbf{Digit generation}: We conduct an experiment on images to verify that both architectures can learn higher-dimensional distributions. We select the digit images as described in Section~\ref{sec:polygan_linear_mnist}. In this case, \modelone{} is implemented as follows: each $\bm{U}\matnot{i}$ is a series of linear convolutions with stride $2$ for $i=1, \ldots, 4$, while $\bm{C}$ is a linear residual block. We emphasize that in both models all the activation functions are removed and there is a single $tanh$ in the output of the generator for normalization purposes.

\begin{figure}[h]
    \subfloat[GT]{\includegraphics[width=0.31\linewidth]{linear-mnist_gt.png}\hspace{+0.6mm}}
    \subfloat[\modelone]{\includegraphics[width=0.31\linewidth]{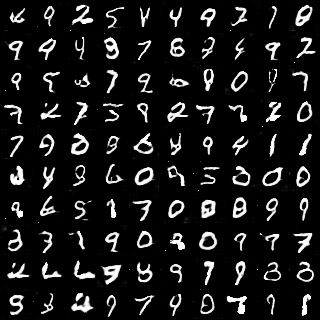}\hspace{+0.6mm}}
    \subfloat[\modeltwo]{\includegraphics[width=0.31\linewidth]{linear-mnist_injection.png}\hspace{+0.1mm}}
\caption{Comparison of the two decompositions on digit generation.}
\label{fig:polygan_linear_mnist_model_comparison}
\end{figure}

\end{document}